\documentclass[lettersize,journal]{IEEEtran}

\usepackage[T1]{fontenc}
\usepackage{acronym}
\usepackage{amsmath,amssymb}
\usepackage{amsthm}
\usepackage{booktabs} 
\usepackage{cases}
\usepackage{cite}
\usepackage{CJK}
\usepackage{color}
\usepackage{float}
\usepackage{graphicx}
\usepackage[linesnumbered,ruled,vlined]{algorithm2e}
\usepackage{multirow}
\usepackage{optidef}
\usepackage{cleveref}
\usepackage{svg}
\usepackage{subcaption}
\usepackage{tabularx}
\usepackage{tcolorbox}
\usepackage{makecell}
\usepackage{pifont}

\SetKwInput{KwInput}{Input}              
\SetKwInput{KwOutput}{Output}

\newtheorem{assumption}{Assumption}

\newtheorem{definition}{Definition}
\newtheorem{lemma}{Lemma}

\newtheorem{problem}{Problem}
\newtheorem{proposition}{Proposition}

\newtheorem{theorem}{Theorem}
\newtheorem{observation}{Observation}

\acrodef{cnn}[CNN]{convolutional neural network}
\acrodef{dp}[DP]{differential privacy}
\acrodef{dpfl}[DP-FL]{differentially private federated learning}
\acrodef{fedavg}[FedAvg]{federated averaging}
\acrodef{fedsgd}[FedSGD]{federated stochastic gradient descent}
\acrodef{fl}[FL]{Federated learning}
\acrodef{ic}[IC]{incentive compatibility}
\acrodef{ir}[IR]{individual rationality}
\acrodef{iid}[IID]{independent and identically distributed}
\acrodef{jsam}[JSAM]{joint client selection and privacy compensation mechanism}
\acrodef{non-iid}[non-IID]{non-independent and identically distributed}

\acrodef{sgd}[SGD]{stochastic gradient descent}

\let\oldfrac\frac
\renewcommand{\frac}[2]{%
	\mathchoice
	{\oldfrac{#1}{#2}}
	{#1/#2}
	{#1/#2}
	{#1/#2}
}

\begin{document}
\title{JSAM: Privacy Straggler-Resilient Joint Client Selection and Incentive Mechanism Design in Differentially Private Federated Learning}

\author{Ruichen~Xu, Ying-Jun Angela Zhang, ~\IEEEmembership{Fellow, ~IEEE}, Jianwei Huang, ~\IEEEmembership{Fellow, ~IEEE} 
	\thanks{An earlier version of this paper was presented in part at the IEEE INFOCOM workshops 2024~\cite{10620900}.
		Ruichen Xu and Ying-Jun Angela Zhang are with the Department of
		Information Engineering, The Chinese University of Hong Kong, Hong Kong (e-mail: xr021@ie.cuhk.edu.hk; co-corresponding author, email: yjzhang@ie.cuhk.edu.hk).
		Jianwei Huang is with the School of Science and Engineering, Shenzhen Institute of Artificial Intelligence and Robotics for Society, Shenzhen Key Laboratory of Crowd Intelligence Empowered Low-Carbon Energy Network, and CSIJRI Joint Research Centre on Smart Energy Storage, The Chinese University of Hong Kong, Shenzhen, Guangdong, 518172, P.R. China. (co-corresponding author, email: jianweihuang@cuhk.edu.cn).}}
	
\maketitle
\begin{abstract}
	Differentially private federated learning faces a fundamental tension: privacy protection mechanisms that safeguard client data simultaneously create quantifiable privacy costs that discourage participation, undermining the collaborative training process. Existing incentive mechanisms rely on unbiased client selection, forcing servers to compensate even the most privacy-sensitive clients (``privacy stragglers''), leading to systemic inefficiency and suboptimal resource allocation. We introduce JSAM (Joint client Selection and privacy compensAtion Mechanism), a Bayesian-optimal framework that simultaneously optimizes client selection probabilities and privacy compensation to maximize training effectiveness under budget constraints. 
	Our approach transforms a complex 2N-dimensional optimization problem into an efficient three-dimensional formulation through novel theoretical characterization of optimal selection strategies. We prove that servers should preferentially select privacy-tolerant clients while excluding high-sensitivity participants, and uncover the counter-intuitive insight that clients with minimal privacy sensitivity may incur the highest cumulative costs due to frequent participation. Extensive evaluations on MNIST and CIFAR-10 demonstrate that JSAM achieves up to 15\% improvement in test accuracy compared to existing unbiased selection mechanisms while maintaining cost efficiency across varying data heterogeneity levels.
\end{abstract}


\section{Introduction}
Federated learning (FL) is a widely studied learning paradigm that facilitates collaborative model training among distributed clients without directly sharing their local data \cite{mcmahan2017communication}.\acused{fl}
During the iterative \ac{fl} training process, each client computes model updates based on their local datasets and shares these updates, typically in the form of model parameter gradients, with a central server. 
The server then aggregates these model updates to iteratively refine the global model.
Despite its decentralized nature, \ac{fl} is susceptible to privacy leaks through various attacks, such as membership inference attacks \cite{melis2019exploiting,shokri2017membership} and data reconstruction attacks \cite{zhu2019deep}.

In recent years, differential privacy (DP) has emerged as a promising privacy protection mechanism in \ac{fl}, offering rigorous privacy guarantees \cite{mcmahan2017learning}, as demonstrated in Fig. \ref{fig:dpfl}.
Specifically, clients add noise to their gradients in accordance with their maximum tolerance to privacy leakage, known as privacy budgets, before transmitting them to the server.
While DP can provide provable privacy protection, it does not fully eliminate privacy leakage \cite{sun2022profit}.
Consequently, clients incur privacy costs that, without adequate compensation mechanisms, may discourage them from participating in \ac{fl}.
\begin{figure}
	\centering
	\includegraphics[width=1\linewidth]{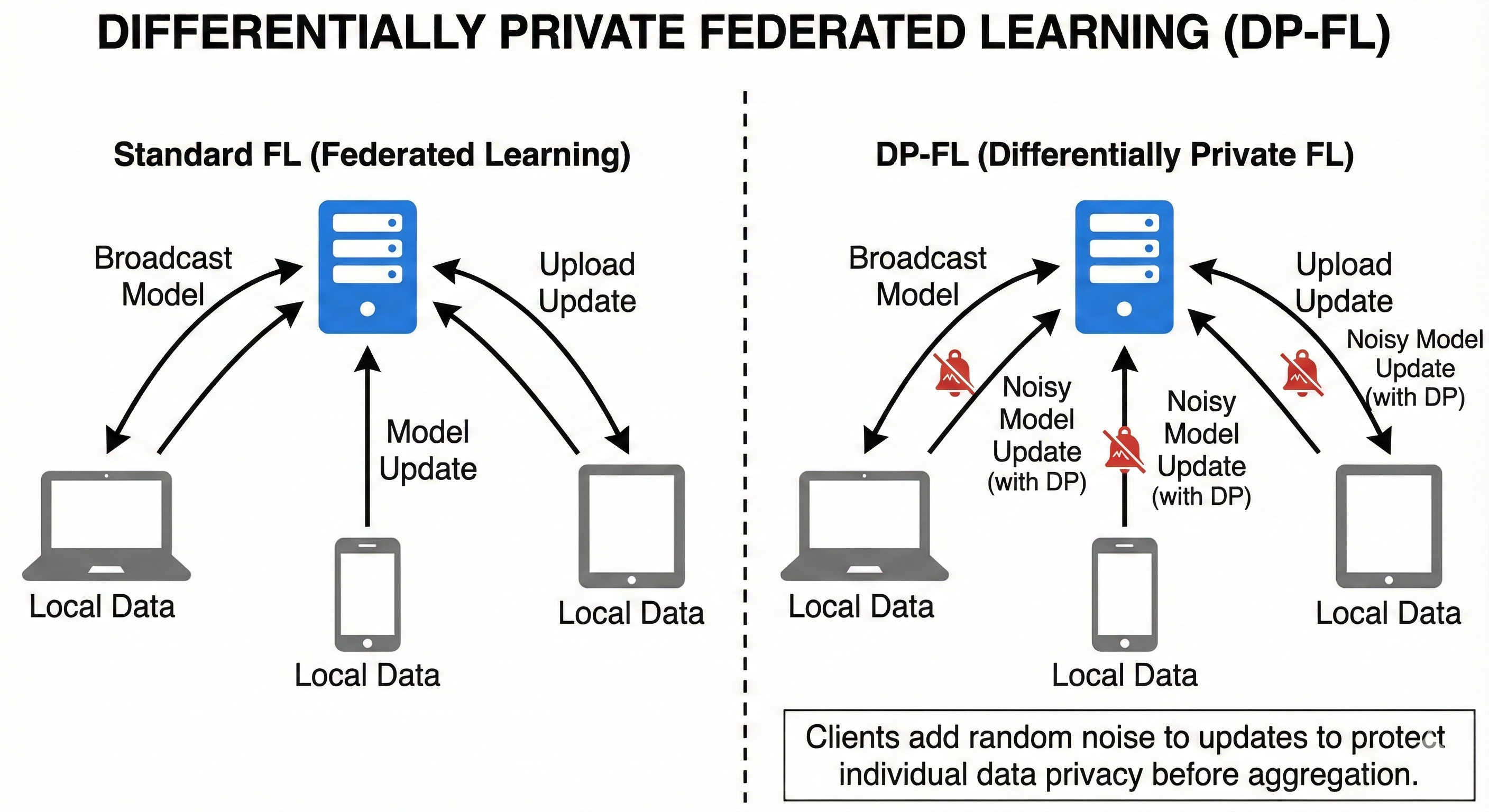}
	\caption{Differentially private federated learning employs noisy gradient updates to safeguard client privacy.}
	\label{fig:dpfl}
\end{figure}

This privacy-participation tension necessitates careful design of client selection and privacy compensation mechanisms in differentially private federated learning (DP-FL).\acused{dpfl} These mechanisms determine which clients to incentivize and how to compensate them for privacy leakage, with the server remunerating clients to secure specific amounts of privacy budgets tailored to their privacy sensitivities. These decisions are inherently interdependent as a client's participation frequency increases, their privacy leakage and associated costs also increase. 

However, most existing studies treat client selection and privacy compensation as disparate problems. 
This decoupled approach is suboptimal, particularly in the presence of \emph{``privacy stragglers''}—clients with stringent privacy requirements. 
In a disjointed design, the aggregation process is often constrained by these stragglers' low privacy budgets, forcing the injection of excessive noise. 
As illustrated in Fig. \ref{fig:straggler}, this privacy straggler effect can significantly degrade global model performance. This necessitates jointly designing client selection and privacy compensation to balance contributions with costs, thereby mitigating straggler effects.

Such a joint design presents several technical challenges. The first challenge lies in modeling how client selection strategies and privacy compensation mechanisms affect the server's payoff, particularly in the presence of non-independent and identically distributed (non-IID) data and heterogeneous privacy sensitivities. \acused{non-iid}
A client's contribution to the global model does not monotonically increase with participation frequency, as frequent participation may introduce bias into the model. Moreover, due to the iterative nature of FL, a client's privacy cost depends on both their sensitivity to privacy leakage and participation frequency. This leads to our first key research question:

\begin{figure}
	\centering
	\includegraphics[width=1\linewidth]{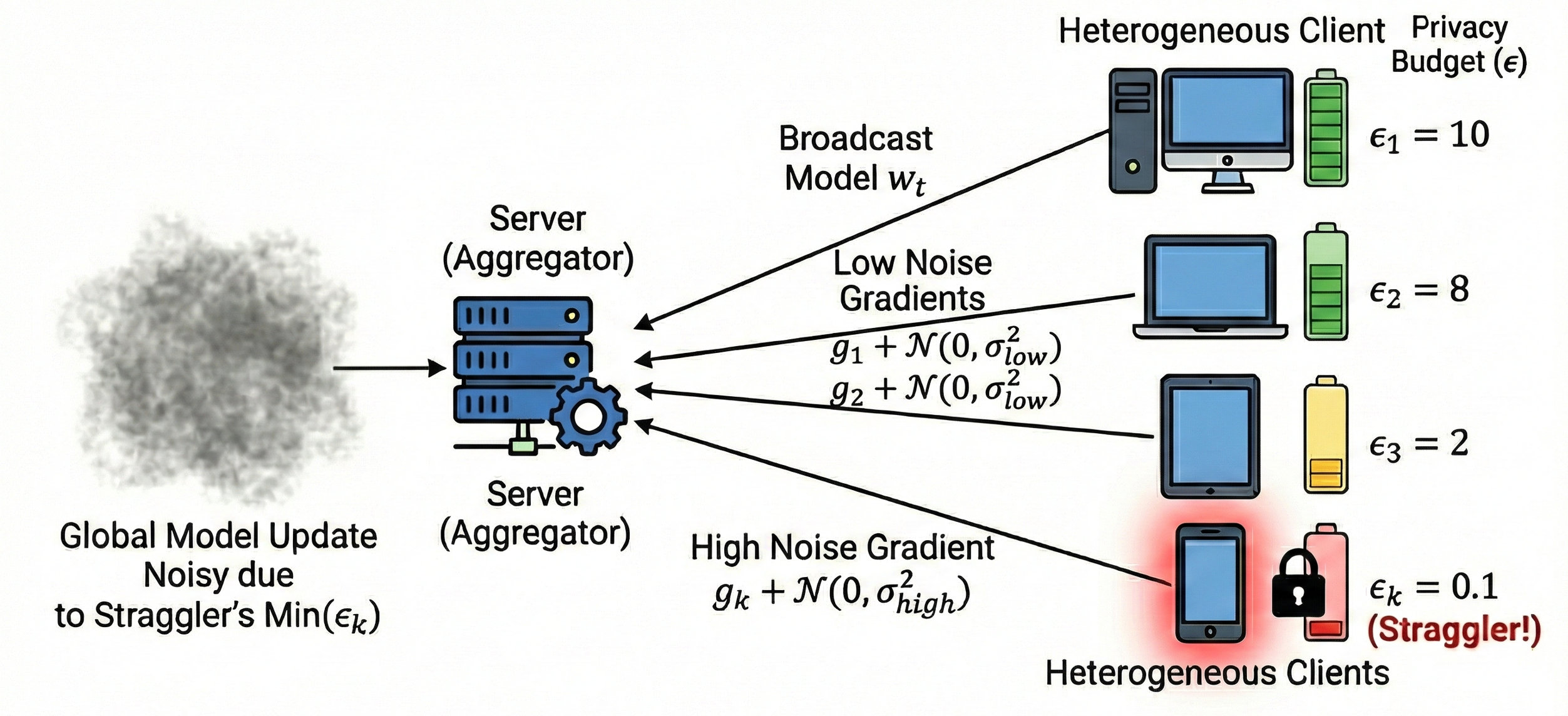}
	\caption{\textbf{The "privacy straggler" effect.} A single client with very low privacy budget ($\epsilon_k$) forces the server to add excessive noise, degrade the model utility.}
	\label{fig:straggler}
\end{figure}

\textbf{Question 1}: \textit{How can the server optimally balance the privacy cost compensations against the degradation in model convergence from biased client selections?}

The second major challenge involves identifying and incentivizing the clients who can contribute most effectively to \ac{fl}. Privacy sensitivities represent private information, creating opportunities for clients to misreport their sensitivities to obtain higher compensation. 
This strategic behavior motivates our second key research question:

\textbf{Question 2}: \textit{How can we design an optimal joint client selection and privacy compensation (incentive) mechanism when clients have private heterogeneous privacy sensitivities in \ac{dpfl}?}

\subsection{Contributions}
The main contributions of the paper are summarized as follows.

\begin{itemize}
\item\textit{Joint client selection and privacy compensation mechanism design}: We present JSAM (\underline{J}oint client \underline{S}election and privacy compens\underline{A}tion \underline{M}echanism), the first work for privacy straggler-resilient joint client selection and privacy compensation design in \ac{dpfl}, considering \ac{non-iid} data and heterogeneous privacy sensitivities.
JSAM addresses the fundamental optimization challenge of balancing clients' data contributions against privacy leakage risks to discard privacy stragglers, thereby boosting the overall privacy-utility tradeoff.

\item\textit{Optimal solution characterization}: Despite the inherent non-convexity of the optimization problem, we successfully characterize the optimal selection probabilities for clients with heterogeneous privacy sensitivities. 
This theoretical result significantly reduces the problem's dimensionality to just three dimensions, enabling efficient algorithmic implementation.

\item\textit{Mechanism design insights}:
Our analysis reveals that optimal selection favors clients with lower privacy sensitivities while potentially excluding high-sensitivity clients entirely. 
We identify an interesting phenomenon where frequently selected clients with low sensitivity may incur higher total privacy costs due to their frequent participation and consequently larger privacy budget allocations.

\item\textit{Comprehensive empirical validation}:
Through extensive numerical evaluations, we show that JSAM outperforms both widely studied unbiased client selection mechanisms and existing biased client selection approaches. 
The empirical results demonstrate that our method achieves up to 15\% improvement in test accuracy while maintaining reasonable monetary costs across different data heterogeneity scenarios.
\end{itemize}

\subsection{Related Work}
This section reviews three key research areas: general FL incentive mechanisms, DP-FL techniques, and incentive mechanisms for DP-FL. 
We identify critical gaps in jointly optimizing client selection and privacy compensation under heterogeneous privacy sensitivities with rigorous convergence analysis.
A detailed comparison is provided in TABLE \ref{tab:comparison}.
\begin{table*}[t]
	\centering
	\small
	\caption{Comparison of Incentive Mechanisms in DP-FL}
	\label{tab:comparison}
	\begin{tabular}{c c c c c}
		\toprule
		\textbf{Mechanism} & \makecell{Client selection} & \makecell{Performance characterization} & \makecell{Client heterogeneity} & \makecell{Is privacy straggler-resilient?} \\
		\midrule
		\cite{sun2021pain,yi2022stackelberg,lin2023heterogeneous,xu2023personalized,hu2020trading,mao2024game,lu2025piece,sun2024socially,sun2022profit} & Unbiased & Convergence analysis & Privacy budget & $\times$ \\
		\cite{wu2021incentivizing} & Unbiased & Convergence analysis & Local dataset size & $\times$ \\
		\cite{qin2024participation} & Unbiased & Convergence analysis & Privacy budget & $\times$ \\
		\cite{xu2021incentive} & Unbiased & Curve fitting & Privacy budget & $\times$ \\
		\cite{ying2020double} & Bid based & Training cost & Training cost & $\times$ \\
		\midrule
		\textbf{Our Method} & \makecell{Performance-aware\\(Flexible)} & Convergence analysis & Privacy budget & $\checkmark$ \\
		\bottomrule
	\end{tabular}
\end{table*}

\subsubsection{Incentive Mechanism Design for Federated Learning}
Encouraging high-quality participation in federated learning is critical, as clients incur costs for computation, communication, and data sharing. Consequently, a significant body of research has focused on designing effective incentive mechanisms. Auction theory is a prominent approach for efficiently allocating resources. 
For instance, some works propose truthful procurement auctions to minimize social cost by co-optimizing client selection, participation scheduling, and the number of training rounds \cite{pang2022incentive}. Others have developed online randomized auctions to manage long-term budget and energy constraints while achieving low regret \cite{yuan2021incentivizing}. 
To address concerns about both data quality and privacy in bidding, multi-dimensional reverse auctions have been used to select participants based on data size, privacy levels, and data distribution \cite{wang2022privaim}.
Recent advancements extend this to dynamic environments, such as incentive-aware federated bargaining for client selection, which integrates auction-based rewards with adaptive bargaining to handle volatile participation \cite{ljv2025incentive}.

Beyond auctions, other economic models have been explored. Contract theory is leveraged to handle information asymmetry, where the server offers a menu of reward options to incentivize clients with multi-dimensional private information \cite{ding2020optimal, liu2021privacy}. 
Additionally, Stackelberg game models have been employed to ensure incentive compatibility in heterogeneous settings, where the server acts as a leader adjusting decay factors to balance client contributions over time, promoting fairness without compromising convergence \cite{javaherian2025incentive}.
Quality-aware approaches employ Earth Mover's Distance for client selection and Stackelberg games for time-dependent rewards, optimizing privacy budgets and server costs in a Nash equilibrium framework \cite{yuan2024qi}.
Other approaches focus on fairness and optimal resource allocation. The Shapley value, for example, has been used to design mechanisms that reward clients commensurate with their contribution to the model's performance \cite{pene2023incentive}. To maximize efficiency under a fixed budget, some mechanisms derive a closed-form optimal sampling probability to minimize the convergence bound \cite{liao2024optimal}. 
These works provide a robust foundation for incentivizing participation but do not typically address the specific challenges introduced by formal privacy guarantees like differential privacy.

Emerging directions also include incentive mechanisms for dynamic client selection in resource-constrained settings, where rewards are tied to real-time contribution metrics to mitigate dropout risks \cite{lu2025piece}.
Personalized incentive mechanisms further integrate clustering preferences to boost participation in heterogeneous data settings, enhancing model appeal and accuracy \cite{khan2023ip}.

Our work considers incentive mechanisms to compensate the clients's privacy costs, which is orthogonal to works considering other aspects such as communication costs and computation costs.

\subsubsection{Differentially Private Federated Learning}
The integration of differential privacy (DP) into \ac{fl} is essential for providing formal privacy guarantees but introduces a fundamental privacy-utility trade-off. \acused{dp}
Much of the research in DP-FL aims to mitigate the negative impact of privacy-preserving noise on model accuracy. 
One line of work focuses on novel noise-injection techniques. 
For example, researchers have proposed using temporally correlated noise in online settings to improve utility \cite{zhang2025locally} or cascading and offsetting noise between training iterations to reduce its cumulative effect \cite{wang2025codp}. 
Adaptive mechanisms dynamically adjust noise levels based on gradient norms or historical updates, enabling layer-specific clipping thresholds to balance privacy and convergence while preserving model fidelity \cite{talaei2024adaptive, cui2025aldp}.
Another strategy involves modifying the communication protocol and aggregation rule, such as using a randomized interaction scheme where clients probabilistically transmit a neighbor's model to reduce their privacy exposure \cite{zhu2025randomized}. 
Tang et al. \cite{tang2025enforcing} developed ReFL, a reputation-based system designed to identify and mitigate the impact of ``rogue'' clients who selfishly inject excessive DP noise to enhance their own privacy, thereby degrading global model accuracy.

While these methods advance the state-of-the-art in privacy preservation, they do not consider the economic incentives required to persuade clients to accept the associated privacy costs.

\subsubsection{Incentive Mechanism Design in Differentially Private Federated Learning}
Combining incentive design with DP-FL presents unique challenges, as compensation must account for clients' heterogeneous privacy sensitivities. Existing research in this area can be broadly categorized into two approaches.

The first approach relies on unbiased client selection, where all clients are potential participants \cite{sun2021pain,yi2022stackelberg,lin2023heterogeneous,xu2023personalized,wu2021incentivizing,hu2020trading}. 
More recently, game-theoretic models have been used to analyze the selfish privacy-preserving behaviors of clients in cross-silo FL, with incentives designed to steer clients from inefficient Nash equilibria toward a socially optimal strategy \cite{mao2024game,qin2024participation}.
In the unbiased client selection paradigm, the mechanism must be designed to compensate even the privacy stragglers. This often leads to systemic inefficiency, as the high rewards required to engage these stragglers may not be justified by their data contribution, straining the server's budget.

To improve efficiency, a second research thrust has emerged based on biased (or strategic) client selection. 
These mechanisms aim to exclude clients who present a poor trade-off between contribution and cost. For example, mechanisms have been developed using the training cost itself as a selection criterion \cite{ying2020double}. 

Current mechanisms fail to jointly optimize client selection and privacy compensation under non-IID data with heterogeneous privacy sensitivities. 
The literature lacks convergence analysis that accounts for the interaction between privacy noise and client selection. 
Unlike existing approaches that treat selection and compensation as separate problems, JSAM integrates rigorous convergence analysis to simultaneously determine optimal selection probabilities and privacy compensation, addressing the fundamental challenge of balancing privacy costs against data utility.

\section{Preliminary}
In this section, we first introduce the basics of \ac{fl} and \ac{dp} in Sections \ref{sec: fl} and \ref{sec: dp}, respectively.
Then, we describe the details of \ac{dp}-\ac{fl} in Section \ref{sec: dp fl}.

\subsection{Federated Learning}\label{sec: fl}
\ac{fl} is a collaborative machine learning paradigm where a central server aims to train a global model using data distributed across $N$ privacy-aware clients, denoted as $\mathcal{N} = \{1,2 ,\cdots, N\}$.
Each client $k\in\mathcal{N}$ possesses a local \ac{non-iid} dataset $\mathcal{M}_k = \{(x_{k,i}, y_{k,i})\mid i \in 1, \cdots, N_k\}$.
For simplicity, we assume that all clients have equal dataset sizes, \textit{i.e.}, $|\mathcal{M}_1| = |\mathcal{M}_2| = \cdots = |\mathcal{M}_N| = M$ .
Our results can be extended to scenarios with imbalanced local datasets.

When all clients participate in the learning, the objective is to minimize the global loss function \cite{li2020federated}:
\begin{align}\label{equ: unbiased obj}
	F(\boldsymbol{w}) = \sum_{k=1}^N \frac{|\mathcal{M}_k|}{\sum_{i=1}^{N}|\mathcal{M}_i|}F_k(\boldsymbol{w}),
\end{align}
where $F_k(\boldsymbol{w})$ is the local objective function for client $k$, defined as:
\begin{align}
	F_k(\boldsymbol{w}) = \frac{1}{|\mathcal{M}_k|} \sum_{(x,y)\in \mathcal{M}_k}\mathcal{L}(x,y, \boldsymbol{w}),
\end{align}
and $\mathcal{L}(x,y,\boldsymbol{w})$ represents the loss of prediction from model $\boldsymbol{w}$ on example $(x,y)$.

However, in practical scenarios, only a subset of clients participate in each training round.
When the server employs a general client selection strategy that selects clients with replacement based on the selection probability vector $\boldsymbol{p}^\textnormal{s} = [p^\textnormal{s}_k]_{k=1}^{N}$, the global objective function becomes:
\begin{align}
	F^\textnormal{b}(\boldsymbol{w}) = \sum_{k=1}^Np^\textnormal{s}_kF_k(\boldsymbol{w}).
\end{align}

Most existing studies (\textit{e.g.}, \cite{sun2021pain},\cite{mcmahan2017communication}) consider an unbiased client selection strategy where $\boldsymbol{p}^\textnormal{s}=\boldsymbol{p}^\textnormal{u} = [\frac{|\mathcal{M}_1|}{\sum_{i=1}^{N}|\mathcal{M}_i|},\cdots, \frac{|\mathcal{M}_N|}{\sum_{i=1}^{N}|\mathcal{M}_i|}]$, ensuring that the objective function aligns with the global objective in equation (\ref{equ: unbiased obj}).
However, this unbiased selection becomes suboptimal in scenarios where clients exhibit heterogeneous privacy sensitivities. 

Our proposed mechanism is applicable to both cross-device and cross-silo FL settings, as the underlying algorithmic principles are consistent across these scenarios.

\subsection{Differential Privacy}\label{sec: dp}
\ac{dp} is a privacy notion that ensures no adversary can distinguish two neighboring datasets from the algorithm's output \cite{dwork2006differential}. 
\ac{dp} possesses desirable properties, such as composability and robustness to post-processing \cite{dwork2014algorithmic}. 
The definition of \ac{dp} is given as follows.

\begin{definition}[Differential privacy]
	A randomized algorithm $\mathcal{M} : \mathcal{X} \rightarrow \mathcal{R}$ is said to satisfy ($\epsilon,\delta$)-\ac{dp} if for any neighboring datasets $X, X' \in \mathcal{X}$  that differ in at most one entry, and for all subsets $\mathcal{S}\subseteq \mathcal{R}$, the following holds:
	\begin{align}
		\Pr\left[\mathcal{M}(X)\in\mathcal{S}\right]\le e^\epsilon \Pr\left[\mathcal{M}(X')\in\mathcal{S}\right]+\delta.
	\end{align}
\end{definition}
A prevalent approach to achieve \ac{dp} is adding Gaussian noise to sensitive information before its release.
Specifically, when a private dataset is utilized iteratively $T_0$ times, the noise variance of the Gaussian noise is determined as outlined in Lemma \ref{lemma: MA} \cite{abadi2016deep}.

\begin{lemma}[Noise variance\cite{abadi2016deep}]\label{lemma: MA}
	There exist constants $c_1$ and $c_2$ so that for any $\epsilon < c_1T_0$, adding $\mathcal{N}(0,\sigma^2)$ noise on the output in each iteration is $(\epsilon, \delta)$-differentially
	private for any $\delta > 0$ if
	\begin{align}
		\sigma \ge c_2 \frac{\sqrt{T_0\log(1/\delta)}}{\epsilon},
	\end{align}
where $T_0$ is the number of iterations.
\end{lemma}

\subsection{Differentially Private Federated Learning}\label{sec: dp fl}
In this subsection, we introduce the workflow of \ac{dp}-\ac{fl}.
Before training, the server determines the client selection strategy $\boldsymbol{p}^\text{s}$ and the noise variance $\boldsymbol{\sigma}^2$.
Subsequently, clients and server iteratively exchange noisy gradients and perform model aggregations.

The \ac{dp}-\ac{fl} training process operates as follows:

\begin{enumerate}
	\item Client selection and model Distribution: 
	The server determines the subset of clients $\mathcal{S}_t$ to be selected in each iteration $t$ based on the selection probabilities $\boldsymbol{p}^\textnormal{s}$ and dispatches the current global model $\boldsymbol{w}_t$ to the selected clients.
	\item Local gradient computation and perturbation: Upon receiving $\boldsymbol{w}_t$, each selected client $k \in \mathcal{S}_t$ performs the following steps:
	\begin{itemize}
		\item Gradient computation: Computes the gradient of the local loss function with respect to the model parameters, $\nabla \mathcal{L}_{\boldsymbol{w}_t}(x, y, \boldsymbol{w}_t)$.
		\item Gradient Clipping: Applies an L2-norm clipping function to the computed gradient to bound its sensitivity:
		\begin{align}
			\boldsymbol{g}_{k,t}^\text{clip} = \frac{1}{|\mathcal{M}_k|}\!\!\sum_{(x,y)\in \mathcal{M}_k}\!\!\!\! \text{Clip}\!\left(\nabla \mathcal{L}_{\boldsymbol{w}_t}(x,y,\boldsymbol{w}_t),C\right),
		\end{align}
		where $\text{Clip}(\boldsymbol{z},C)$ is a clipping function that clips the L2-norm of $\boldsymbol{z}$ to $C$, \textit{i.e.}, $\text{Clip}(\boldsymbol{z},C) = \frac{\boldsymbol{z}}{\max\{1,\|\boldsymbol{z}\|_2/C\}}$.
		\item Noise Addition: Adds Gaussian noise to the clipped gradient to ensure DP:
		\begin{equation}
			\begin{aligned}
			\boldsymbol{g}_{k,t} \!=& \boldsymbol{g}_{k,t}^\text{clip}+ \mathcal{N}(0, \sigma_k^2C^2\boldsymbol{I}),
			\end{aligned}
		\end{equation}
		
		According to Lemma \ref{lemma: MA}, the noise variance for client $k$ is given by: \begin{align} \sigma_k^2 = c_2^2 T_k \frac{\log(1/\delta)}{\epsilon_k^2}, \end{align} where $T_k$ denotes the total number of times client $k$ has been selected during FL, and $\epsilon_k$ represents the privacy budget of client $k$.
	\end{itemize}
	\item Model aggregation and update: The server aggregates the noisy gradients received from the selected clients using a learning rate $\alpha$: 
\begin{align}
	\boldsymbol{w}_{t+1} = \boldsymbol{w}_t - \alpha \frac{1}{|\mathcal{S}_t|}\sum_{k\in\mathcal{S}_t}\boldsymbol{g}_{k,t},
\end{align}
	 The updated global model $\boldsymbol{w}_{t+1}$ is then disseminated to clients in the subsequent iteration.
	 \item Termination: After $T$ iterations, the server finalizes the trained model as $\boldsymbol{w}_T$.
\end{enumerate}
Pre-determined client selection is a fundamental requirement in DP-FL due to the privacy composition theorem. Under differential privacy, each access to a client's data consumes part of the privacy budget, and the total privacy cost accumulates across all training rounds. Consequently, the noise variance required to achieve a target privacy level $(\epsilon,\delta)$ must be calibrated based on the total number of times each client participates throughout the entire training process. This necessitates fixing all client selection probabilities $\{p_i\}_{i=1}^N$ before training begins, as any adaptive selection based on intermediate results would violate the privacy composition bounds.

\section{System Model}\label{sec: model}
In this section, we model the payoff functions for the clients and the server taking into account privacy sensitivities, payments, and model accuracy.

\subsection{Client Payoff}
For any client $k \in \mathcal{N}$, their privacy cost is directly proportional to their allocated privacy budget \cite{ghosh2011selling}, following the established privacy-cost model:
\begin{align}
	\textnormal{cost}_k^\textnormal{p} = c_k\epsilon_k,
\end{align} 
where $c_k$ represents the privacy sensitivity type of client $k$ (the per unit cost of privacy leakage) and $\epsilon_k$ denotes the privacy budget allocated to client $k$.
To offset these privacy costs, the server implements a compensation mechanism through a payment vector $\boldsymbol{\pi} = [\pi_k]_{k\in\mathcal{N}}$. 
This leads to a straightforward utility function for each client $k$:
\begin{align}
	U^\textnormal{c}_k = \pi_k - c_k\epsilon_k.
\end{align}
\subsection{Server Payoff}
The server's objective encompasses two primary goals: minimizing the model training loss and optimizing the total monetary compensation allocated to participating clients. 
The server's objective can be formally expressed as:
\begin{equation}\label{opt: server}
	\begin{aligned} 
		\min& \text{ } C^\textnormal{s} = \eta\textnormal{Training loss} + \sum_{k \in \mathcal{N}} \pi_k,
	\end{aligned}
\end{equation}
where $\eta$ represents a weighting coefficient that balances the server's priorities between training performance and cost efficiency.

Given the inherent complexity of \ac{fl}, precisely determining individual client contributions to the global model's training loss a priori is computationally intractable. 
To address this challenge, we employ an upper bound approximation of the training loss, as proposed in \cite{xu}.
\begin{theorem}[Convergence analysis\cite{xu}]
	Suppose the objective function is $L$-smooth and the gradient norm is bounded by $R$.
	When employing stepsize of $\alpha = \oldfrac{R}{C}\oldfrac{1}{\sqrt{T}}$, the expected average gradient norm satisfies
	\begin{equation}\label{equ: convergence_smooth}
		\begin{aligned}
			&\sqrt{\frac{1}{T}\sum_{t=1}^{T}\mathbb{E}\left[\left\|\nabla F(\boldsymbol{w}_t)\right\|_2^2\right]}
			\le \frac{R}{2}\underbrace{\left(\left\| \boldsymbol{p}^\textnormal{s}\!-\!\boldsymbol{p}^\textnormal{u}\right\|_1+G^\textnormal{clip}\right)}_\textnormal{Non-vanishing training loss}\\
			&\!+\!\frac{R}{2}\!\underbrace{\left[\sqrt{(\left\| \boldsymbol{p}^\textnormal{s}\!-\!\boldsymbol{p}^\textnormal{u}\right\|_1\!+\!G^\textnormal{clip})^2 \!+\! 2\sum_{k=1}^{N}(p^\textnormal{s}_{k})^2D\sqrt{T}L\frac{\sigma_k^2}{T_k}}\right]}_\textnormal{Non-vanishing training loss}\\
			&+\sqrt{\frac{F(\boldsymbol{w}_1)-F(\boldsymbol{w}^*)}{\sqrt{T}}+\frac{1}{2} \oldfrac{1}{\sqrt{T}}LR^2},
		\end{aligned}
	\end{equation}
	where $G^\textnormal{clip}$ is a term related to the clipping error, term $D$ is the gradient dimension, and $T$ is the total number of training iterations.
\end{theorem}

Following \cite{xu}, assuming the clipping threshold is appropriately chosen and ignoring the clipping error $G^\textnormal{clip}$, we can simplify the non-vanishing training loss in (\ref{equ: convergence_smooth}) as 
\begin{equation}
\begin{aligned}
	& \left\| \boldsymbol{p}^\textnormal{s}-\boldsymbol{p}^\textnormal{u}\right\|_1 + \sqrt{\left\| \boldsymbol{p}^\textnormal{s}-\boldsymbol{p}^\textnormal{u}\right\|_1^2 + Q \sum_{k:p^\textnormal{s}_k\neq 0} \frac{(p^\textnormal{s}_{k})^2}{\epsilon_k^2}},
\end{aligned}
\end{equation}
where $Q = 2c_2^2\log(1/\delta)D\sqrt{T}L$. 

Combining both the training loss and the monetary costs, the server's objective becomes:
\begin{equation}\label{opt: server summarize}
	\begin{aligned} 
		\min \text{ } C^s \!=&\; \eta\left(\left\| \boldsymbol{p}^\textnormal{s}\!-\!\boldsymbol{p}^\textnormal{u}\right\|_1 \!+\!\! \sqrt{\!\left\| \boldsymbol{p}^\textnormal{s}\!-\!\boldsymbol{p}^\textnormal{u}\right\|_1^2 \!+\! Q\! \sum_{k:p^\textnormal{s}_k\neq 0}\! \frac{(p^\textnormal{s}_{k})^2}{\epsilon_k^2}}\right) \\
		&\;+ \sum_{k \in \mathcal{N}} \pi_k.
	\end{aligned}
\end{equation}
This formulation incorporates three key components:
\begin{itemize}
	\item The term $\left\| \boldsymbol{p}^\textnormal{s} - \boldsymbol{p}^\textnormal{u} \right\|_1$
	quantifies the deviation error from an unbiased client selection strategy, which directly contributes to the training loss in the presence of non-IID data. 
	Critically, existing DP-FL incentive mechanisms that rely on convergence analysis have been limited to unbiased selection, and thus have not characterized the error introduced by a more flexible client selection approach.
	\item The term $Q \sum_{k:p^\textnormal{s}_k\neq 0} \frac{(p^\textnormal{s}_{k})^2}{\epsilon_k^2}$ captures the training loss introduced by \ac{dp}.
	\item The term $\sum_{k \in \mathcal{N}} \pi_k$ represents the total monetary cost incurred by the server in compensating clients.
\end{itemize}

In practical implementations, the server typically lacks prior knowledge of clients' privacy sensitivities ($c_k$). To address this challenge, we introduce the Joint client Selection and privacy compensAtion Mechanism (JSAM), designed to effectively elicit these unknown privacy sensitivity parameters from participating clients.

\section{Joint Client Selection and Privacy Compensation Mechanism Design}\label{sec: incentive mechanism}
In this section, we begin by introducing the JSAM framework in Section \ref{subsec: jsam framework}. 
Next, we describe the JSAM problem formulation in Section \ref{sec: opt problem}, followed by its reformulation in Section \ref{sec: reformulation}.

\subsection{JSAM}\label{subsec: jsam framework}
\begin{figure}
	\centering
	\includegraphics[width=1\linewidth]{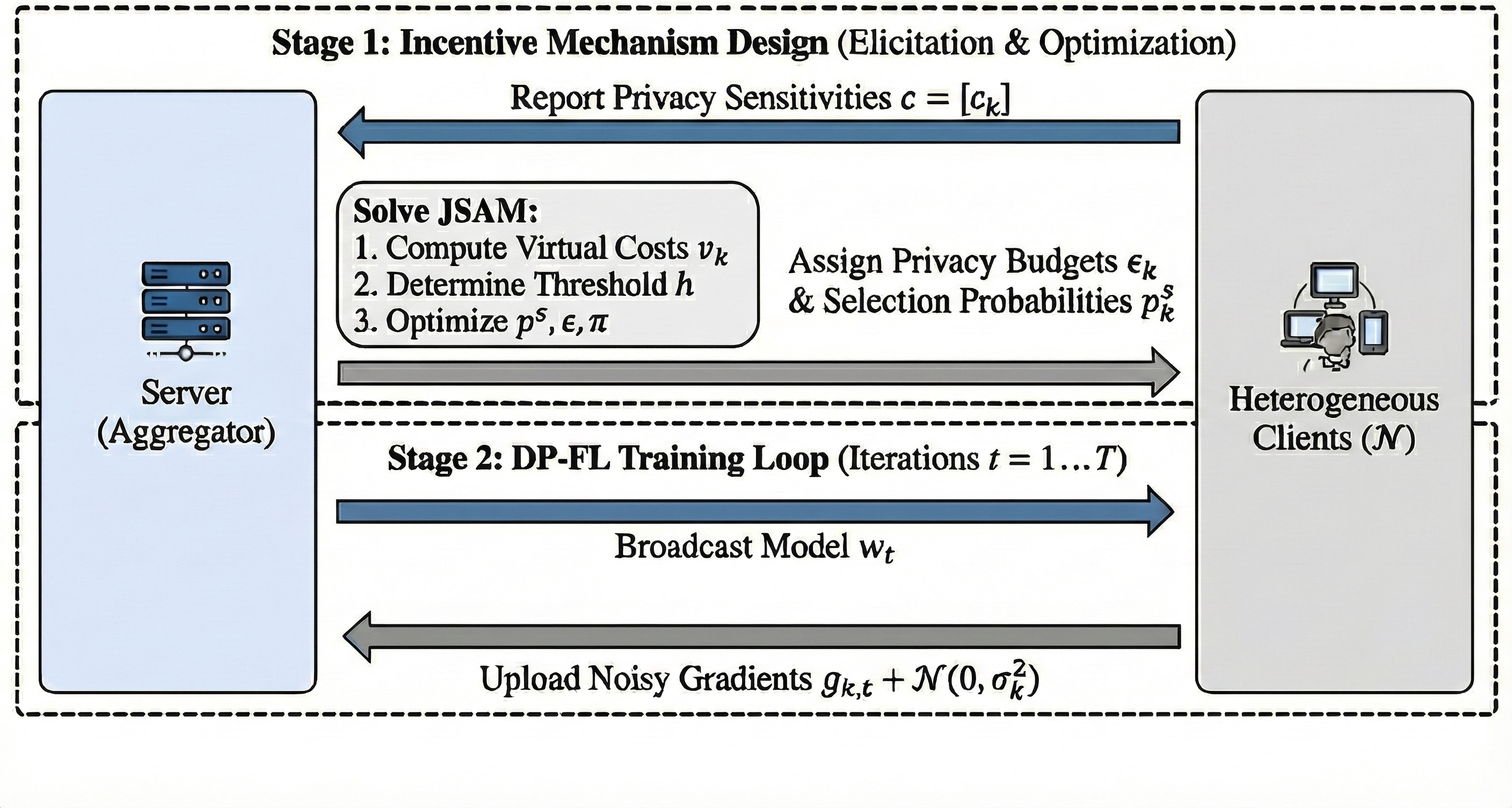}
	\caption{Workflow illustration of JSAM.}
	\label{fig:diagram}
\end{figure}

The proposed JSAM, illustrated in Fig. \ref{fig:diagram}, establishes a two-stage process where the server optimizes the triple $(\boldsymbol{p}^\textnormal{s}, \boldsymbol{\epsilon},\boldsymbol{\pi})$ based on clients' reported privacy types $\boldsymbol{c}$ to maximize its payoff.
The mechanism begins with an elicitation phase, where clients submit their privacy sensitivity reports $\boldsymbol{c} = [c_k]_{k\in\mathcal{N}}$. Subsequently, the server utilizes these reports to determine the optimal configuration for the privacy budgets $\boldsymbol{\epsilon} = [\epsilon_k]_{k\in\mathcal{N}}$, selection probabilities $\boldsymbol{p}^\textnormal{s} = [p^\textnormal{s}_k]_{k\in\mathcal{N}}$, and compensation payments $\boldsymbol{\pi} = [\pi_k]_{k\in\mathcal{N}}$.

Recalling from Section \ref{sec: dp fl} that the client selection plan should be pre-determined, the compensation model is simplified to a single payment. 
This payment is issued once to cover a client's privacy costs for the entire duration of the training.

The subsequent subsections provide a detailed analysis of the server-side optimization process for these parameters.
\subsection{Problem Formulation}\label{sec: opt problem}
In our mechanism, the server initially lacks knowledge of clients' privacy sensitivities $\boldsymbol{c}$ and must elicit this information directly from the clients. 
According to the revelation principle \cite{myerson1981optimal}, we can focus on direct mechanisms where truthful reporting of clients' types constitutes a Bayesian Nash equilibrium.
 
The mechanism must satisfy two fundamental constraints: incentive compatibility (IC) and individual rationality (IR). \acused{ic} \acused{ir}
\begin{definition}[Incentive compatibility - IC]
	A mechanism satisfies incentive compatibility when each client maximizes its payoff through truthful reporting of its privacy sensitivity.
	Formally, for all $k\in\mathcal{N}$ and for all $c_k,c_k' \!\in\! \mathbb{R}^+$,
	\begin{align}\nonumber
		\mathbb{E}_{\boldsymbol{c}_{-k}}\left[U_k^\textnormal{c} (c_k, \boldsymbol{c}_{-k})\right]\! \ge\! \mathbb{E}_{\boldsymbol{c}_{-k}}\left[U_k^\textnormal{c} (c_k', \boldsymbol{c}_{-k})\right].
	\end{align}
\end{definition}
The \ac{ir} constraint ensures that each participating client receives a non-negative payoff, \textit{i.e.}, 
\begin{definition}[Individual rationality - IR]
	A mechanism satisfies individual rationality when each participating client receives a non-negative expected payoff. 
	Formally, for all $k\in\mathcal{N}$ and for all $c_k \in \mathbb{R}^+$,
	\begin{align}\nonumber
		\mathbb{E}_{\boldsymbol{c}_{-k}}\left[U_k^\textnormal{c}(c_k,\boldsymbol{c}_{-k})\right] \ge 0.
	\end{align}
\end{definition}

\subsubsection{JSAM Design Problem}
The server's objective is to design optimal selection and payment rules $(\boldsymbol{p}^\textnormal{s}, \boldsymbol{\epsilon},\boldsymbol{\pi})$ that maximize its payoff while satisfying both IC and IR constraints. This can be formalized as:
\begin{tcolorbox}[left = 0.5pt,top=0.5pt,bottom=0.5pt]
\begin{problem}{JSAM design problem.}\label{prob: opt}

\begin{equation}\nonumber
	\begin{aligned}
		\min_{\boldsymbol{\epsilon(\boldsymbol{c})}, \boldsymbol{p}^\textnormal{s}(\boldsymbol{c}), \boldsymbol{\pi}(\boldsymbol{c})}&\text{ }\!\mathbb{E}_{\boldsymbol{c}}\left[\eta \left\| \boldsymbol{p}^\textnormal{s}(\boldsymbol{c})-\boldsymbol{p}^\textnormal{u}\right\|_1+ \sum_{k \in \mathcal{N}}\pi_k(\boldsymbol{c})\right]\\
		 &\!+ \mathbb{E}_{\boldsymbol{c}}\left[\eta\sqrt{\left\| \boldsymbol{p}^\textnormal{s}(\boldsymbol{c})-\boldsymbol{p}^\textnormal{u}\right\|_1^2 + Q\!\! \sum_{k:p^\textnormal{s}_k\neq 0}\!\! \frac{(p^\textnormal{s}_{k}(\boldsymbol{c}))^2}{\epsilon_k^2(\boldsymbol{c})}}\right]\\
		\textup{s.t. }\,\,\,\,\,\,\,&\! \sum_{k=1}^{N}p^\textnormal{s}_{k}(\boldsymbol{c}) = 1, p^\textnormal{s}_{k}(\boldsymbol{c}) \ge 0, \forall k \in \mathcal{N},\\
		&\!  \mathbb{E}_{\boldsymbol{c}_{-k}}\!\!\left[\pi_k(c_k, \boldsymbol{c}_{-k}\!)  \!-\!c_k\epsilon_k(c_k,\boldsymbol{c}_{-k})\right]\! \ge\! 0, \forall k,\!c_k,\!c_k',\\
		&\! \mathbb{E}_{\boldsymbol{c}_{-k}}\left[\pi_k(c_k,\boldsymbol{c}_{-k}) \!-\!c_k\epsilon_k(c_k,\boldsymbol{c}_{-k})\right]\\
		&\!\ge\mathbb{E}_{\boldsymbol{c}_{-k}}\left[\pi_k(c_k',\boldsymbol{c}_{-k}) -c_k\epsilon_k(c_k',\boldsymbol{c}_{-k})\right],\forall k,c_k.
	\end{aligned}
\end{equation}
\end{problem}
\end{tcolorbox}

\subsection{Reformulating the JSAM Design Problem}\label{sec: reformulation}

Directly solving Problem \ref{prob: opt} is challenging due to the complexity introduced by the payment function $\boldsymbol{\pi}(\boldsymbol{c})$, which depends on the distribution of privacy sensitivities $\boldsymbol{c}$.
To make the problem tractable, we use a series of reformulations to it, shown is Fig. \ref{fig:workflow}.
We begin by characterizing the set of payments that satisfy both \ac{ic} and \ac{ir} constraints.
\begin{lemma}\label{lemma: ir ic}
	For any client $k \in \mathcal{N}$, a payment $\pi_k$ satisfies the \ac{ir} and \ac{ic} constraints if and only if there exists a function $\epsilon_k(c_k, \boldsymbol{c}_{-k})$ that is weakly decreasing with $c_k$ and
	\begin{align}\label{equ: payment}
		\pi_k(c_k) = \int_{c_k}^{\infty}\mathbb{E}_{\boldsymbol{c}_{-k}}\epsilon_{k}(z,\boldsymbol{c}_{-k})dz + c_k\epsilon_{k}(c_k,\boldsymbol{c}_{-k}) + d_k,
	\end{align}
	where $d_k$ is a non-negative constant.
\end{lemma} 
We defer the proof of Lemma \ref{lemma: ir ic} in Appendix \ref{sec: mechanism design proof}.
Based on Lemma \ref{lemma: ir ic}, we substitute the payment function (\ref{equ: payment}) to the objective function of Problem \ref{prob: opt}.
This substitution introduces a dependence on the distribution of privacy sensitivities through the virtual cost, defined as follows:
\begin{definition}[Virtual cost]
	The virtual cost of client $k\in \mathcal{N}$ is defined as
	\begin{align}
		v_k = c_k + \frac{F_k(c_k)}{f_k(c_k)},
	\end{align}
	where $F_k(\cdot)$ and $f_k(\cdot)$ are the cumulative density function and the probability density function of client $k$'s privacy sensitivity $c_k$, respectively.
\end{definition}
This leads to our reformulated optimization problem:

\begin{tcolorbox}[left = 0.5pt,top=0.5pt,bottom=0.5pt]
\begin{problem}{Reformulated JSAM design problem.}\label{prob: reformulated}
	
	\begin{mini!}|s|[2]  
		{\scriptsize\boldsymbol{\epsilon},\boldsymbol{p}^\textnormal{s}}         
		{\!\!\!\eta\!\left(\!\left\| \boldsymbol{p}^\textnormal{s}\!-\!\boldsymbol{p}^\textnormal{u}\right\|_1\! \!+\!\! \sqrt{\!\left\| \boldsymbol{p}^\textnormal{s}\!-\!\boldsymbol{p}^\textnormal{u}\right\|_1^2 \!+\! Q\!\!\! \sum_{k:p^\textnormal{s}_k\neq 0}\!\!\! \frac{(p^\textnormal{s}_{k})^2}{\epsilon_k^2}}\right)\!\!+\!\!\sum_{k\in \mathcal
				N}\!\epsilon_kv_k } 
		{\label{equ: re obj}}           
		{}                                
		\addConstraint{\sum_{k\in \mathcal{N}}p^\textnormal{s}_{k} = 1, p^\textnormal{s}_{k}, \epsilon_k \ge 0, \forall k \in \mathcal{N}}{\label{equ: re con1}}  
		\addConstraint{\epsilon_k \textnormal{ weakly decreases with }c_k, \forall k \in \mathcal{N}.}{\label{equ: re con2}}
	\end{mini!}
\end{problem}
\end{tcolorbox}
\begin{figure}
	\centering
	\includegraphics[width=0.5\linewidth]{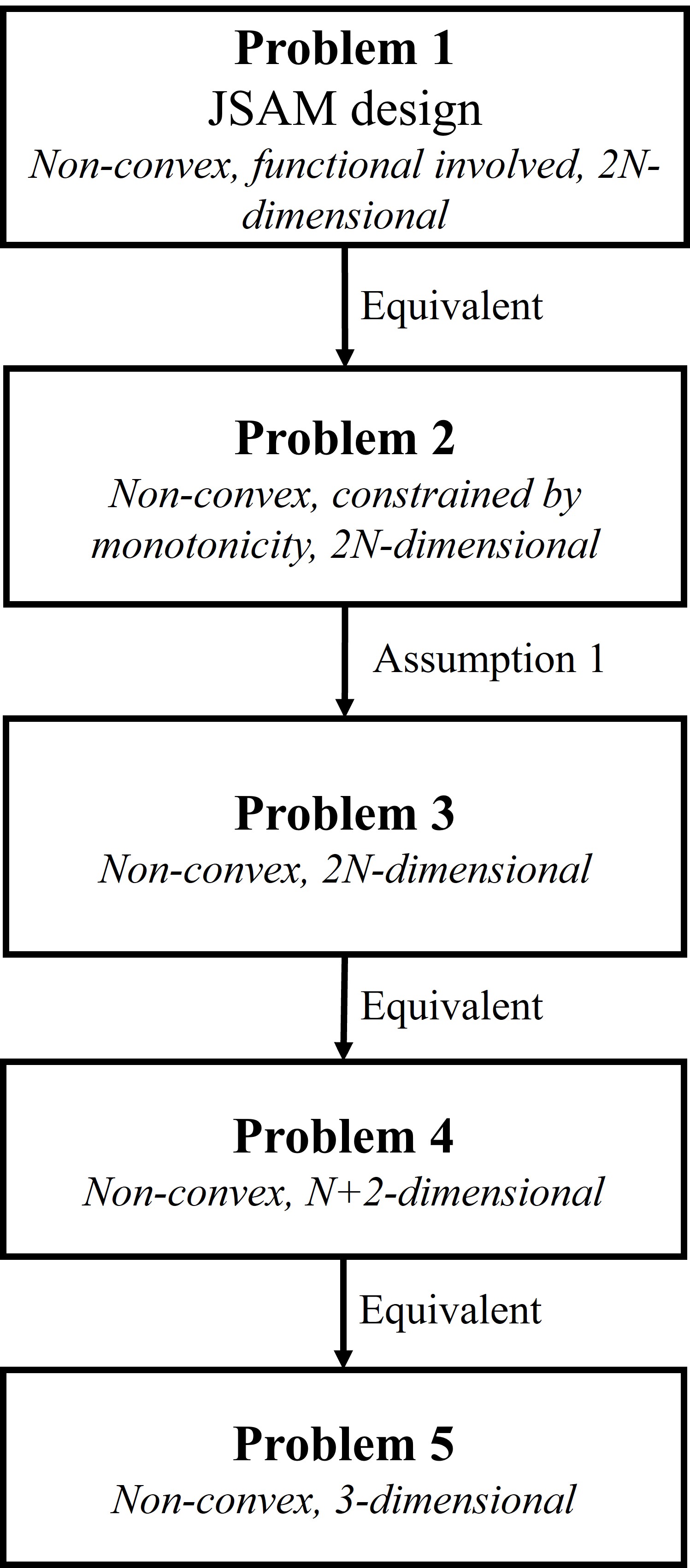}
	\caption{Flow chart of problem reformulations.}
	\label{fig:workflow}
\end{figure}
\begin{theorem}\label{theorem: reformulation}
Given privacy sensitivities $\boldsymbol{c}$, Problems \ref{prob: opt} and \ref{prob: reformulated} are equivalent, sharing the same optimal solutions. 
\end{theorem}
We show the proof of Theorem \ref{theorem: reformulation} in Appendix \ref{sec: mechanism design proof}.
The additional term $\frac{F_k(c_k)}{f_k(c_k)}$ in the virtual cost arises from the \ac{ic} constraint. 
This term reflects the information asymmetry between the server and clients. Since the server does not know the clients' privacy sensitivities, it must incur additional costs to incentivize truthful reporting.
The constraint (\ref{equ: re con2}) ensures that the server adjusts the privacy budget $\epsilon_k$ to penalize any misreporting by clients.

Problem \ref{prob: reformulated} remains a challenging non-convex optimization problem due to the coupling between the selection probabilities $\boldsymbol{p}^\textnormal{s}$ and the privacy budgets $\boldsymbol{\epsilon}$.
However, as we will demonstrate in Section \ref{sec: solve}, it is possible to characterize the optimal client selection strategy and develop an efficient algorithm to solve the problem despite its non-convexity.

\section{Solving the Reformulated JSAM Design Problem}\label{sec: solve}

This section analyzes Problem \ref{prob: reformulated}'s optimal solution by relaxing constraint (\ref{equ: re con2}) while maintaining solution equivalence through specific assumptions about virtual costs. This approach yields a more tractable optimization problem while preserving key insights about privacy sensitivity and budget allocation relationships.

\subsection{Problem Relaxation and Equivalence}
We begin by relaxing constraint (\ref{equ: re con2}) in Problem \ref{prob: reformulated} while preserving essential solution properties through the following fundamental assumption:
\begin{assumption}\label{assumption: increasing}[Regularity condition]
	For any $k \in \mathcal{N}$, the virtual cost $v_k$ increases with the privacy sensitivity $c_k$. 
\end{assumption}
Assumption \ref{assumption: increasing} is standard and widely adopted in the mechanism design literature \cite{fallah2022optimal,krishna2009auction}.
It encompasses distributions with log-concave functions, such as uniform and Gaussian distributions \cite{fallah2022optimal}.
Under this assumption, we formulate the relaxed optimization problem:
\begin{tcolorbox}[left = 0.5pt,top=0.5pt,bottom=0.5pt]
	\begin{problem}\label{prob: relaxed}
		\begin{equation}\label{opt:relaxed}\nonumber
			\begin{aligned}
				\min_{\boldsymbol{\epsilon}, \boldsymbol{p}^\textnormal{s}}& \text{ }\eta \left(\!\left\| \boldsymbol{p}^\textnormal{s}-\boldsymbol{p}^\textnormal{u}\right\|_1\! +\! \sqrt{\left\| \boldsymbol{p}^\textnormal{s}-\boldsymbol{p}^\textnormal{u}\right\|_1^2 \!+ Q\! \sum_{k:p^\textnormal{s}_k\neq 0} \frac{(p^\textnormal{s}_{k})^2}{\epsilon_k^2}}\right)\\
				&+ \sum_{k \in \mathcal{N}}\epsilon_kv_k\\
				\text{ }\textup{s.t.}& \sum_{k\in \mathcal{N}}p^\textnormal{s}_{k} = 1, p^\textnormal{s}_{k} , \epsilon_k \ge 0, \forall k \in \mathcal{N}.
			\end{aligned}
		\end{equation}
	\end{problem}
\end{tcolorbox}

\begin{proposition}\label{proposition: blue}
	Under Assumption \ref{assumption: increasing}, the optimal solution to Problem \ref{prob: relaxed} satisfies the condition that $\epsilon_k^*$ is weakly decreasing with $c_k$.
\end{proposition}
The significance of Proposition \ref{proposition: blue} lies in demonstrating that the relaxed Problem \ref{prob: relaxed} inherently maintains the crucial monotonicity property of the privacy budget allocation, making it equivalent to the original Problem \ref{prob: reformulated} for practical purposes.
We defer its proof to Appendix \ref{subsec: blue}.

\subsection{Characterization of the Optimal Solution}

Due to the inequality constraints and non-continuity of the L1 norm term $\left\|\boldsymbol{p}^\textnormal{s}-\boldsymbol{p}^\textnormal{u}\right\|_1$ in Problem \ref{prob: relaxed}, we classify the clients into the four distinct categories: $\mathcal{S}^+, \mathcal{S}^\textnormal{u}, \mathcal{S}^-, \mathcal{S}^0$.
This classification facilitates a structured analysis of the optimal solution.
\begin{itemize}
\item $\mathcal{S}^+$ denotes the set of clients with selection probabilities larger than $\frac{1}{N}$, \textit{i.e.}, $\mathcal{S}^+ = \{k\mid p_k^\textnormal{s} > \frac{1}{N} \}$. 

\item $\mathcal{S}^\textnormal{u}$ denotes the set of clients with selection probabilities $\frac{1}{N}$, \textit{i.e.}, $\mathcal{S}^+ = \{k\mid p_k^\textnormal{s} = \frac{1}{N} \}$.

\item $\mathcal{S}^-$ denotes the set of clients with selection probabilities less than $\frac{1}{N}$, \textit{i.e.}, $\mathcal{S}^- = \{k \mid 0 < p^\textnormal{s}_k < \frac{1}{N}\}$.

\item $\mathcal{S}^0$ denotes the set of clients with selection probabilities $0$, \textit{i.e.}, $\mathcal{S}^0 = \{k\mid p^\textnormal{s}_k=0\}$. 
\end{itemize}
\begin{figure}[t]
	\centering
	\includegraphics[width=0.9\linewidth]{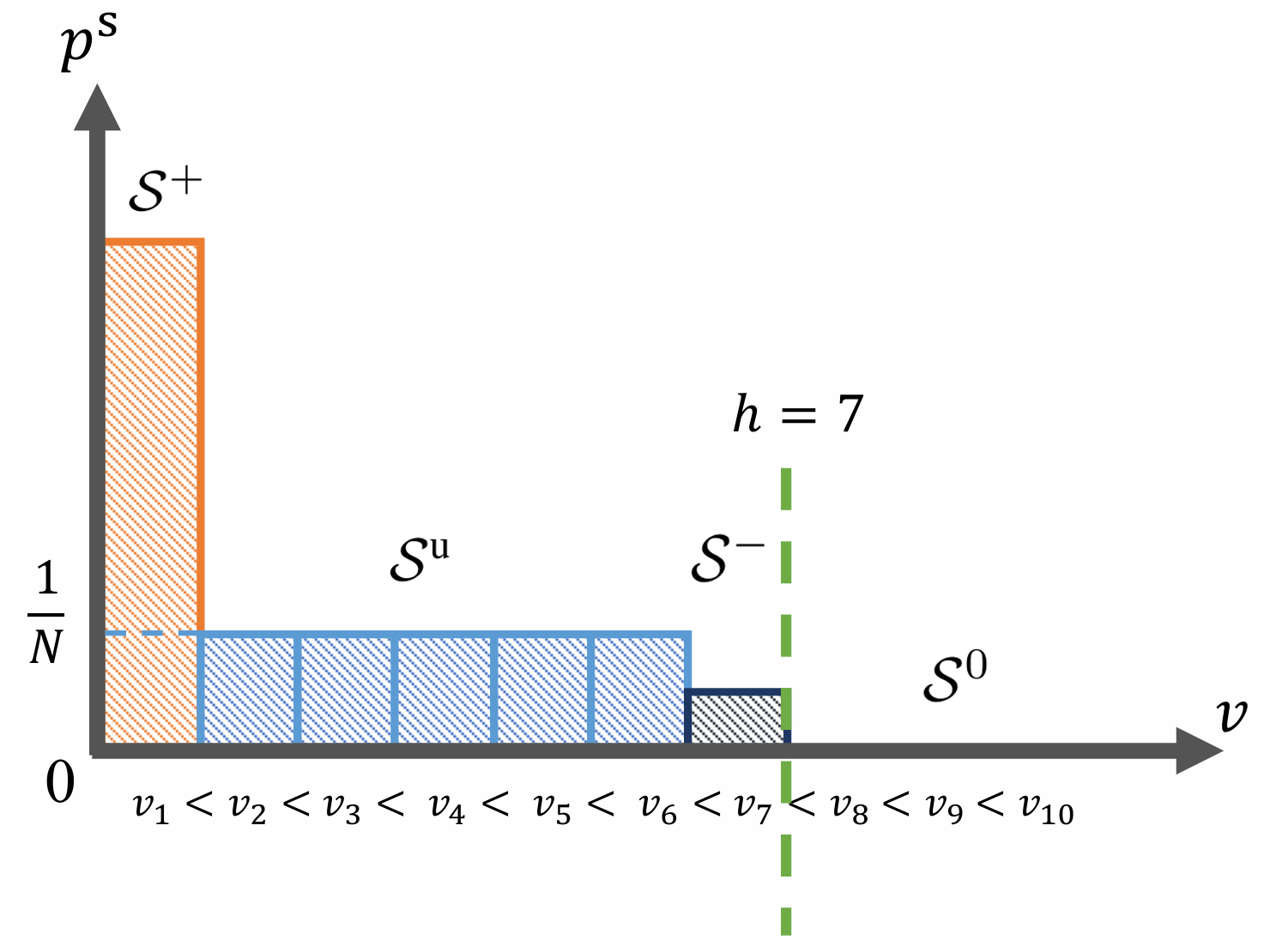}
	\caption{An illustration of the structure of the optimal client selection probabilities with ten clients $(v_1 < \cdots < v_{10})$.}
	\label{fig:probdis}
\end{figure}
We then characterize the optimal client selection $\boldsymbol{p}^{\textnormal{s}*}$ and the corresponding optimal client sets $\mathcal{S}^{+*}, \mathcal{S}^{\textnormal{u}*}, \mathcal{S}^{-*}$, and $\mathcal{S}^{0*}$ in Theorem \ref{theorem: structure}.
Fig. \ref{fig:probdis} illustrates the structure of the optimal solution. 

\begin{theorem}\label{theorem: structure}
	Suppose that we order the clients according to an increasing order of the virtual costs, \textit{i.e.}, $v_1 \le v_2 \le \cdots \le v_N$.
	There exists a threshold $h\in \mathcal{N}$ such that the optimal selection probability $\boldsymbol{p}^{\textnormal{s}*} = [p_k^{\textnormal{s}*}]_{k\in \mathcal{N}}$ of Problem \ref{prob: relaxed} satisfies\footnote{Here $\mathcal{Z}^+$ refers to the set of non-negative integers.} 
	\begin{subnumcases}{p_k^{\textnormal{s}*}} \label{equ: prob_dis 1}
		\ge 1/N, &\text{if } $k=1$,\\\label{equ: prob_dis 2}
		= 1/N, &\text{if } $k \in [2,h-1]\cap \mathcal{Z}^+$, \\ \label{equ: prob_dis 3}
		\le 1/N, &\text{if } $k=h$, \\\label{equ: prob_dis 4}
		= 0, &\text{if } $k\in [h+1,N]\cap \mathcal{Z}^+$.
	\end{subnumcases}
\end{theorem}
Here are the insights of the optimal client selection probabilities $\boldsymbol{p}^{\textnormal{s}*}$ (in Theorem \ref{theorem: structure}).
\begin{itemize}
	\item The results (\ref{equ: prob_dis 1}, \ref{equ: prob_dis 2}, \ref{equ: prob_dis 3}) imply that the server selects the client with the smallest virtual cost the most frequently (with a probability over $\frac{1}{N}$). 
	In addition, the server selects the clients with moderate virtual costs (\textit{i.e.}, smaller than $v_h$ and larger than $v_1$) unbiasedly.
	\item The result (\ref{equ: prob_dis 4}) implies that the server excludes the clients with high virtual costs (\textit{i.e.}, larger than $v_h$). 
	This result indicates that the incentive mechanisms with unbiased selection are sub-optimal to decreasing the server's cost when the clients have different privacy sensitivities. 
	\item Solving the client selection problem is equivalent to finding the threshold $h$ and the threshold probability $p^{\textnormal{s}*}_h$.
	This reduces the dimension of the optimal client selection problem from $N$ to $2$.
\end{itemize}

Based on Theorem \ref{theorem: structure}, we can reformulate Problem \ref{prob: relaxed} by focusing on finding the optimal threshold $h^*$, threshold probability $p_h^{\textnormal{s}*}$, and total monetary cost $B^*$. This reformulation significantly reduces the problem's dimensionality, making it more tractable.
\begin{tcolorbox}[left = 0.5pt,top=0.5pt,bottom=0.5pt]
\begin{problem}\label{prob: bilevel}
	\begin{equation}\label{opt: bilevel}
		\begin{aligned}
			\!\min_{h, p^\textnormal{s}_h}		
		\min_{\boldsymbol{\epsilon}}\! &\text{ }\!\eta \!\sqrt{\!4\!\left(p^\textnormal{s}_1\!-\!p_h^\textnormal{s} \right)^2 \!+ \!Q\!\!\left(\frac{(p_1^\textnormal{s})^2}{\epsilon_1^2}\!+\!\!\sum_{k,1< k< h}\frac{1}{N^2\epsilon_k^2}\!+\! \frac{(p^\textnormal{s}_h)^2}{\epsilon_h^2}\!\right)}\!
			\\
			&+ 2\eta \left(p^\textnormal{s}_1-p^\textnormal{s}_h\right)+\sum_{k \in \mathcal{N}}\epsilon_kv_k\\
			\textup{s.t.            }&p^\textnormal{s}_1 = \frac{2\!+\!N\!-\!h\!}{N}-p^\textnormal{s}_h,\\&\epsilon_k \ge 0, \forall k \in \mathcal{N}, 0\le p_h^\textnormal{s} \le \frac{1}{N}, h \in [0,N]\cap \mathcal{Z}^+.
		\end{aligned}
	\end{equation}
\end{problem}
\end{tcolorbox}
The problem of finding the optimal threshold $h^*$ and threshold probability $p_{h}^{\textnormal{s}*}$ can be regarded as a search problem.
Given $h,p_h^\textnormal{s}$ obtained through this search, the next step is to solve the inner problem ($\min_{\boldsymbol{\epsilon}}$) of Problem \ref{prob: bilevel}.
This involves finding the minimum objective value with respect to $\boldsymbol{\epsilon}$ while keeping $h$ and $p_h^\text{s}$ fixed.

\subsection{Solving the Inner Problem of Problem \ref{prob: bilevel}}

We address the inner optimization problem of Problem 4 by establishing the relationship between optimal privacy budgets and monetary costs. Given selection probability vector $\boldsymbol{p}^s$ (determined by $h$ and $p_h^s$), we show that the optimal privacy budgets $\boldsymbol{\epsilon}^*$ are functions of the total monetary cost $B$.

\begin{theorem}\label{theorem: budgets}
Given $\boldsymbol{p}^\textnormal{s},B$, the optimal privacy budget of client $k\in\mathcal{N}$ is
	\begin{align}
		\epsilon^*_k (\boldsymbol{p}^\textnormal{s},B) = \frac{(p_k^\textnormal{s})^{2/3}B}{\left(\sum_{i=1}^{N}v_i^{2/3}(p^\textnormal{s}_i)^{2/3}\right)v_k^{1/3}}.
	\end{align}
\end{theorem}
We provide the proof of Theorem \ref{theorem: budgets} in Appendix \ref{app: budgets}.
The insights of Theorem \ref{theorem: budgets} is
\begin{itemize}
	\item The optimal privacy budget of client $k$ increases with the selection probability $p_k^\text{s}$ and decreases with the virtual cost $v_k$. 
	\item Counter-intuitively, clients with minimal privacy sensitivity may incur the highest monetary costs ($B_k = \epsilon_kv_k$) due to larger privacy budget compensation.
\end{itemize}

Using Theorem \ref{theorem: budgets} and KKT conditions, we express the objective function as follows:
\begin{proposition}\label{proposition: obj}
	With optimal privacy budgets, the objective function of Problem \ref{prob: bilevel} becomes:
	\begin{equation}
	\begin{aligned}
		&f(B,h,p_h^\textnormal{s})=\\
		&\eta\!\left(\!\! \sqrt{4(p_1^\textnormal{s}\!-\!p_h^\textnormal{s})^2 \!+\! \left(\!\!(v_1p_1^\textnormal{s})^{\oldfrac{2}{3}}\!+\!(v_hp_h^\textnormal{s})^{\oldfrac{2}{3}}\!+\!\!\!\!\sum_{1<i<h}\!\frac{v_i^{\oldfrac{2}{3}}}{N^{\oldfrac{2}{3}}}\!\right)^3\!\!/\!B^2\!}\!\right)\\
		&+    2\eta\left(p_1^\textnormal{s}-p_h^\textnormal{s}\right) + B.
	\end{aligned}
	\end{equation}
\end{proposition}
This leads to the reformulated optimization problem:
\begin{tcolorbox}[left = 0.5pt,top=0.5pt,bottom=0.5pt]
\begin{problem}\label{prob: final} 
\begin{equation}\label{opt: P-}
	\begin{aligned}
		\!\min_{h,p^\textnormal{s}_h} \min_{B}&\text{ }\!\eta\!\!\left(\!\! \sqrt{\!4(p_1^\textnormal{s}\!-\!p_h^\textnormal{s})^2 \!\!+\!\! \left(\!\!\!(v_1p_1^\textnormal{s})\!^{\oldfrac{2}{3}}\!\!+\!(\!v_hp_h^\textnormal{s})\!^{\oldfrac{2}{3}}\!\!+\!\!\!\!\!\sum_{1<i<h}\!\!\frac{v_i^{\oldfrac{2}{3}}}{N^{\oldfrac{2}{3}}}\!\right)^3\!\!\!\!/\!B^2\!}\!\right)\\
		&+ 2\eta\left(p_1^\textnormal{s}-p_h^\textnormal{s}\right) + B\\
		\textup{s.t. }
		& p_1^\textnormal{s} = \frac{2+N-h}{N} - p_h^\textnormal{s},\\
		&B\ge 0, 0\le p_h^\textnormal{s} \le \frac{1}{N}, h \in [0,N]\cap \mathcal{Z}^+.
	\end{aligned}
\end{equation}
\end{problem}
\end{tcolorbox}

The optimization structure of Problem \ref{prob: final} reveals that its inner minimization over $B$ exhibits convex properties, making it amenable to efficient solution through standard convex optimization techniques.

This reformulation represents a significant simplification of our original problem: we have successfully transformed the non-convex JSAM design problem from its initial 
$2N$-dimensional form into an equivalent, more tractable 3-dimensional optimization framework.

To solve this reformulated problem, we propose Algorithm \ref{algo: searching}, which operates as follows:

\begin{enumerate}
	\item Initialization (Line 1): Initialize the best objective value $f^*$ to a large number.
	\item Threshold Search (Lines 2-9): Search through feasible combinations of threshold $h$ and probability $p_h^\textnormal{s}$, computing optimal monetary cost $B$ for each pair.
	\item Selection Probability Assignment (Lines 10-20): Assign final selection probabilities and privacy budgets based on Theorems \ref{theorem: structure} and \ref{theorem: budgets}.
\end{enumerate}
The algorithm achieves computational efficiency with complexity $O(1/\Delta)$, where $\Delta$ denotes the search grid granularity.

\renewcommand{\algorithmcfname}{Algorithm}

\begin{algorithm}[t]
	\DontPrintSemicolon
	\caption{JSAM Algorithm}
	\label{algo: searching}
	\KwInput{Virtual costs $v_1, \ldots, v_N$ $(v_1 < v_2 < \cdots < v_N)$, grid $\Delta$, weight $\eta$}
	\KwOutput{$\boldsymbol{p}^*$, $\boldsymbol{\epsilon}^*$}
	
	Initialize $f^* \gets 10^6$
	
	\For{$i = 1, 2, \ldots, N$}{
		$h \gets N + 1 - i$
		
		\For{$m = 0, 1, \ldots, \lfloor 1/(N\Delta) \rfloor$}{
			$p_1^{\text{s}} \gets \dfrac{i}{N} + m\Delta, \quad p_h^{\text{s}} \gets \dfrac{1}{N} - m\Delta$
			
			$c \gets \left( v_1^{2/3} \left(p_1^{\text{s}}\right)^{2/3} + v_h^{2/3} \left(p_h^{\text{s}}\right)^{2/3} + \dfrac{\sum_{1<j<h} v_j^{2/3}}{N^{2/3}} \right)^3$
			
			$B_v \gets \displaystyle\min_B \left[ 2\eta\left(p_1^{\text{s}} - p_h^{\text{s}}\right) + \eta\sqrt{4\left(p_1^{\text{s}} - p_h^{\text{s}}\right)^2 + c/B^2} + B \right]$
			
			Compute objective $f$ with $B = B_v$
			
			\If{$f < f^*$}{
				$f^* \gets f, \quad h^* \gets h, \quad p_h^{\text{s}*} \gets p_h^{\text{s}}, \quad B^* \gets B_v$
			}
		}
	}
	
	\For{$k = 1, 2, \ldots, N$}{
		\uIf{$k = 1$}{
			$p_k^{\text{s}*} \gets \dfrac{2 + N - h^*}{N} - p_{h^*}^{\text{s}*}$
		}
		\uElseIf{$1 < k < h^*$}{
			$p_k^{\text{s}*} \gets \dfrac{1}{N}$
		}
		\uElseIf{$k = h^*$}{
			$p_k^{\text{s}*} \gets p_{h^*}^{\text{s}*}$
		}
		\Else{
			$p_k^{\text{s}*} \gets 0$
		}
	}
	
	\For{$k = 1, 2, \ldots, N$}{
		$\epsilon_k^* \gets \dfrac{\left(p_k^{\text{s}*}\right)^{2/3} B^*}{v_k^{1/3} \sum_{i=1}^{N} v_i^{2/3} \left(p_i^{\text{s}*}\right)^{2/3}}$
	}
	
	\Return{$\boldsymbol{p}^*$, $\boldsymbol{\epsilon}^*$}
\end{algorithm}

\section{Numerical Results}
In this section, we perform numerical simulations to evaluate the performance of the JSAM.
This section is organized as follows: Section \ref{subsec: setup} presents the simulation setup, Section \ref{subsec: baseline} presents the baselines, followed by performance analysis and discussions for Section \ref{subsec: test loss} to Section \ref{subsec: num}.

\subsection{Simulation Setup}\label{subsec: setup}
We evaluate JSAM's performance against baseline mechanisms using MNIST \cite{lecun1998gradient} and CIFAR-10 \cite{krizhevsky2009learning} datasets under realistic non-IID federated settings. Our experiments examine convergence behavior, privacy-utility trade-offs, and mechanism efficiency across varying privacy sensitivity distributions and system parameters.

\subsubsection{Dataset selection and computational constraints}
We use MNIST and CIFAR-10, which represent the current standard for DP-FL evaluation due to fundamental computational constraints. Single-machine DP-FL simulations face two critical bottlenecks: (i) memory limitations from serially loading data for numerous clients, causing cache thrashing that makes the process memory-bound, and (ii) computational overhead from per-example gradient clipping required for differential privacy. These constraints explain why DP-FL research consistently adopts these smaller datasets \cite{zhang2025locally,zhu2025randomized,wang2025codp}.
\subsubsection{Federated data distribution}
We distribute datasets among $N = 100$ clients in a non-IID fashion, allocating 600 samples per client for MNIST and 500 for CIFAR-10. 
The non-IID distribution follows a two-stage process \cite{karimireddy2020scaffold}: each client first receives $s\%$ of data uniformly sampled from the training set, then the remaining $(100-s)\%$ is allocated based on sorted labels, with each client receiving data from at most two distinct classes. We set the data similarity parameter $s\in\{0,30,70,100\}$ to simulate realistic heterogeneity.
\subsubsection{Privacy and mechanism parameters}
Client privacy sensitivities $c_k$ are sampled from  a uniform distribution $\mathcal{U}(0,1)$, following standard mechanism design conventions \cite{fallah2022optimal}. 
We configure $T = 1000$ total iterations, select $10$ clients per round, set clipping threshold $C = 6$, and use differential privacy parameter $\delta = 10^{-5}$.

\subsubsection{Model architectures and implementation}
 The model 
architecture for MNIST consists of a convolutional neural network (CNN) with two convolution layers, followed by $2 \times 2$ max pooling and two fully connected layers with ReLU activation. \acused{cnn}
For CIFAR-10, we adopt the \ac{cnn} architecture from \cite{tramer2020differentially}.
All training uses differentially private SGD via the Opacus framework \cite{yousefpour2021opacus}.

\subsection{Baselines}\label{subsec: baseline}
We consider the following widely studied mechanisms as baselines and compare JSAM with them.
\begin{itemize}
	\item\textbf{Unbiased selection based mechanism (USBM):}
	USBM sets the selection probability of each client proportional to their dataset size ($\boldsymbol{p}^\text{s} = \boldsymbol{p}^\textnormal{u}$) \cite{sun2021pain, lin2023heterogeneous}.
	It incorporates methods with importance compensation, using a compensation parameter to weight each client's update for unbiased learning objectives \cite{liao2024optimal}. 
	The mechanism aims to maximize the server's payoff while satisfying \ac{ic} and \ac{ir} conditions.
	
	\item\textbf{Fixed selection based mechanism (FSBM-$M$):} 
	FSBM selects a fixed subset of $M$ clients with lower privacy costs.
	The mechanism maximizes the server's payoff while \ac{ic} and \ac{ir} conditions are satisfied for the selected clients.
	
	\item \textbf{JSAM with complete information (JSAM-CI):} JSAM-CI operates in an ideal setting where the server has complete knowledge of all clients' privacy costs. 
	The server compensates each client $k \in \mathcal{N}$ based on their true costs $c_k$ rather than virtual costs $v_k$. 
	The mechanism maximizes the server's payoff while satisfying the \ac{ir} condition.
	
	\item \textbf{Biased selection based mechanism (BBM):}
	BBM selects clients in a non-proportional manner relative to their dataset sizes. 
	Given that most biased selection methods don't consider privacy protection, we adopt a widely-studied approach that selects clients with the largest local training losses \cite{cho2022towards}.
	For compatibility, client selection frequency is predetermined before training begins. 
	BBM maximizes the server's payoff while satisfying \ac{ic} and \ac{ir} conditions for the selected clients.
\end{itemize}
\begin{figure*}[t]
	\begin{subfigure}{0.33\linewidth}
		\centering
		\includegraphics[width=1\linewidth]{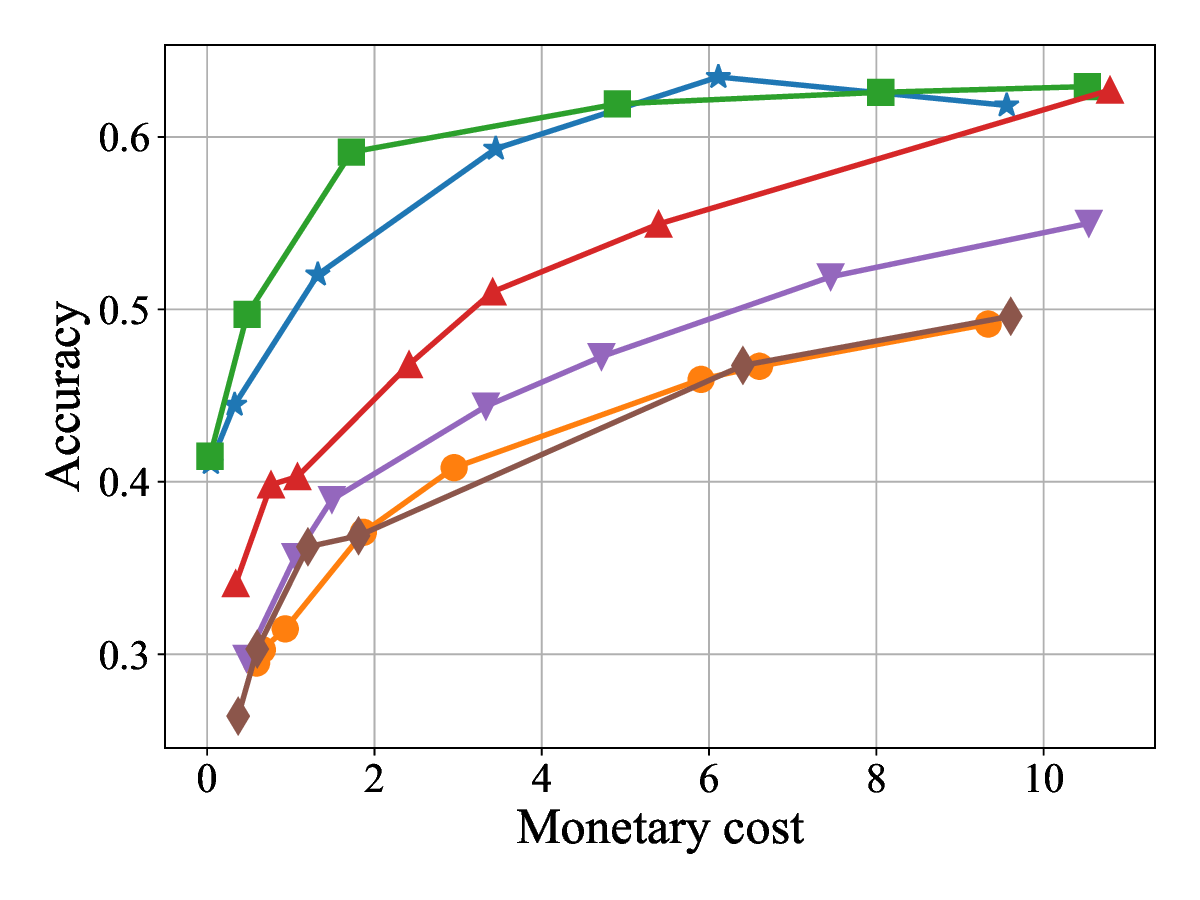}
		\caption{CIFAR-10, Similarity: $s=100$}
		\captionsetup{justification=centering,margin=2cm}
		\label{fig: test_loss_c_iid}
	\end{subfigure}
	\begin{subfigure}{0.33\linewidth}
		\centering
		\includegraphics[width=1\linewidth]{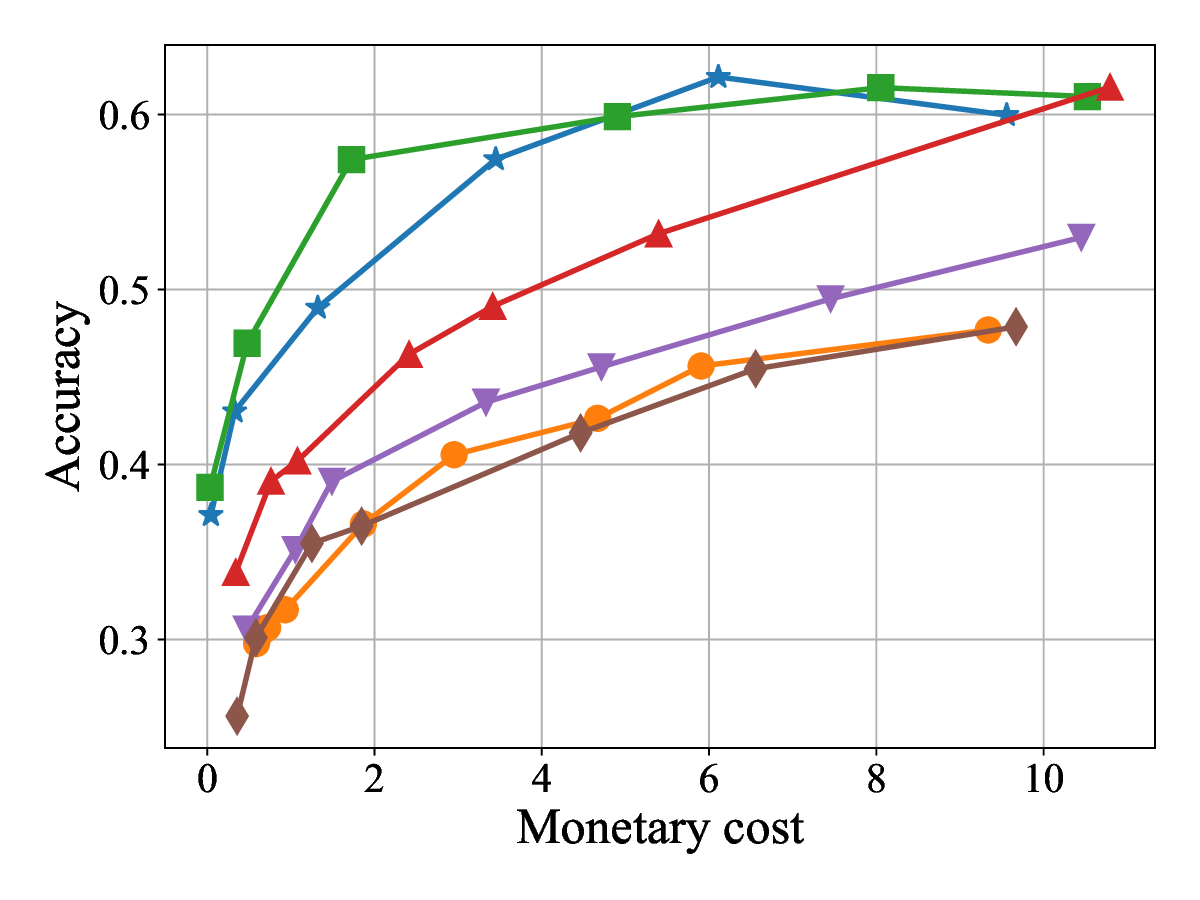}
		\caption{CIFAR-10, Similarity: $s=70$}
		\captionsetup{justification=centering}
		\label{fig: test_loss_c_sim7}
	\end{subfigure}
	\begin{subfigure}{0.33\linewidth}
		\centering
		\includegraphics[width=1\linewidth]{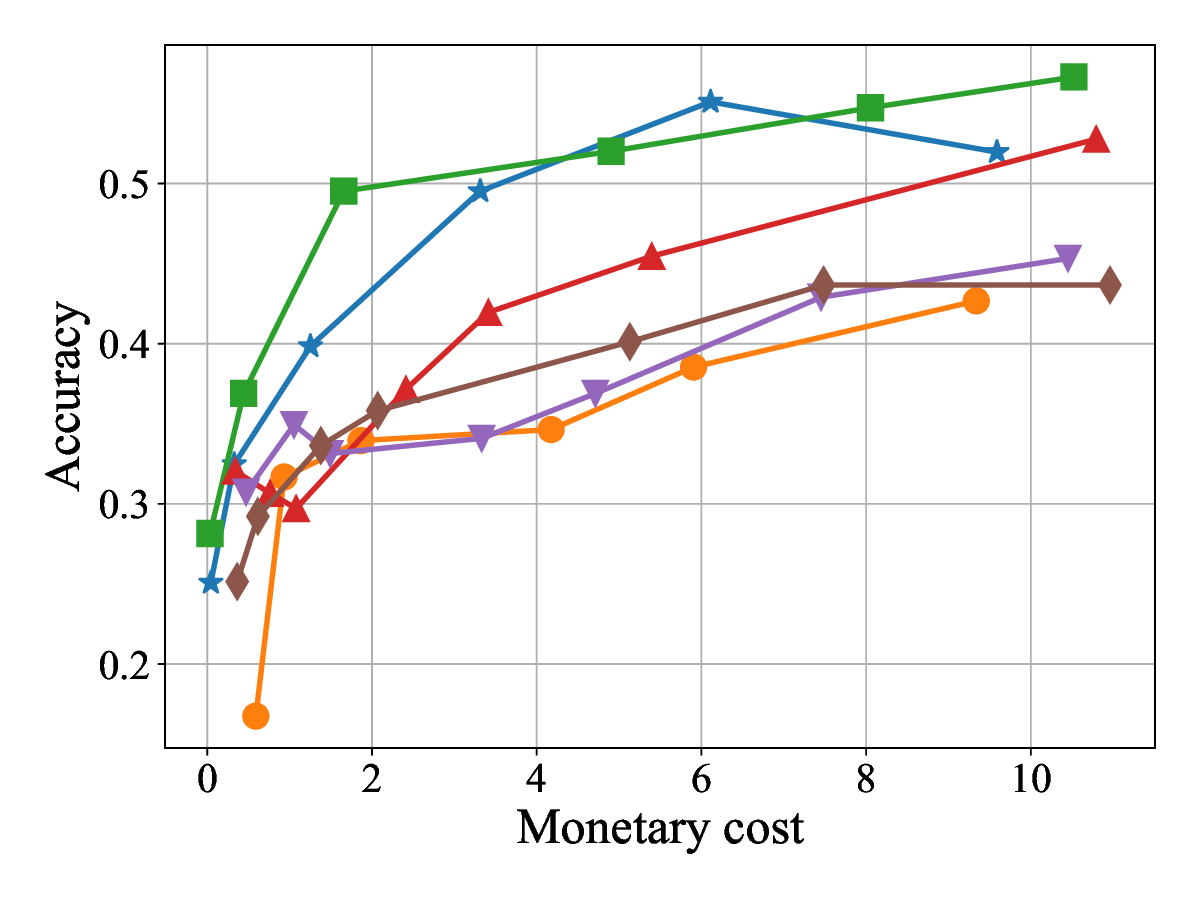}
		\caption{CIFAR-10, Similarity: $s=30$}
		\captionsetup{justification=centering}
		\label{fig: test_loss_c_sim3}
	\end{subfigure}
	\begin{subfigure}{0.33\linewidth}
		\centering
		\includegraphics[width=1\linewidth]{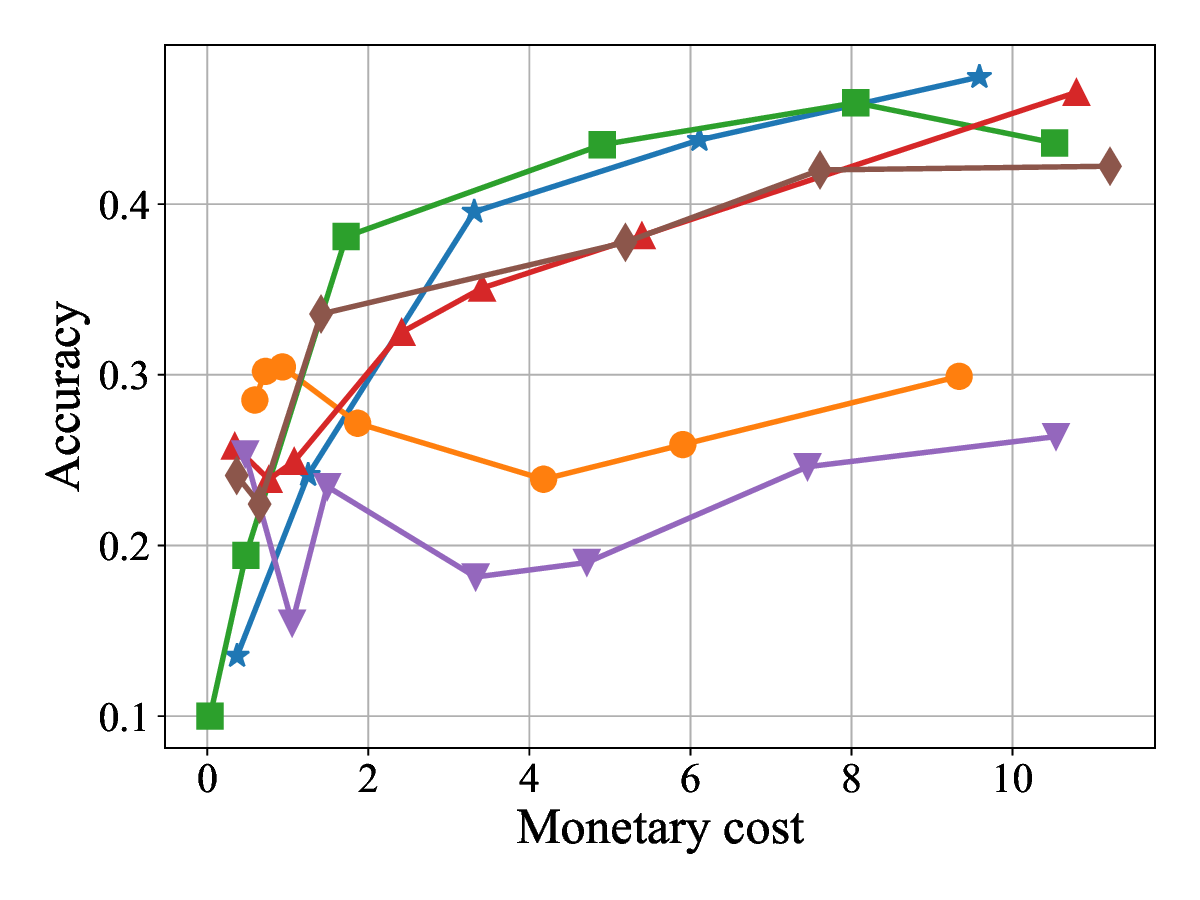}
		\caption{CIFAR-10, Similarity: $s=0$}
		\captionsetup{justification=centering}
		\label{fig: test_loss_c_sim0}
	\end{subfigure}
	\begin{subfigure}{0.33\linewidth}
		\centering
		\includegraphics[width=1\linewidth]{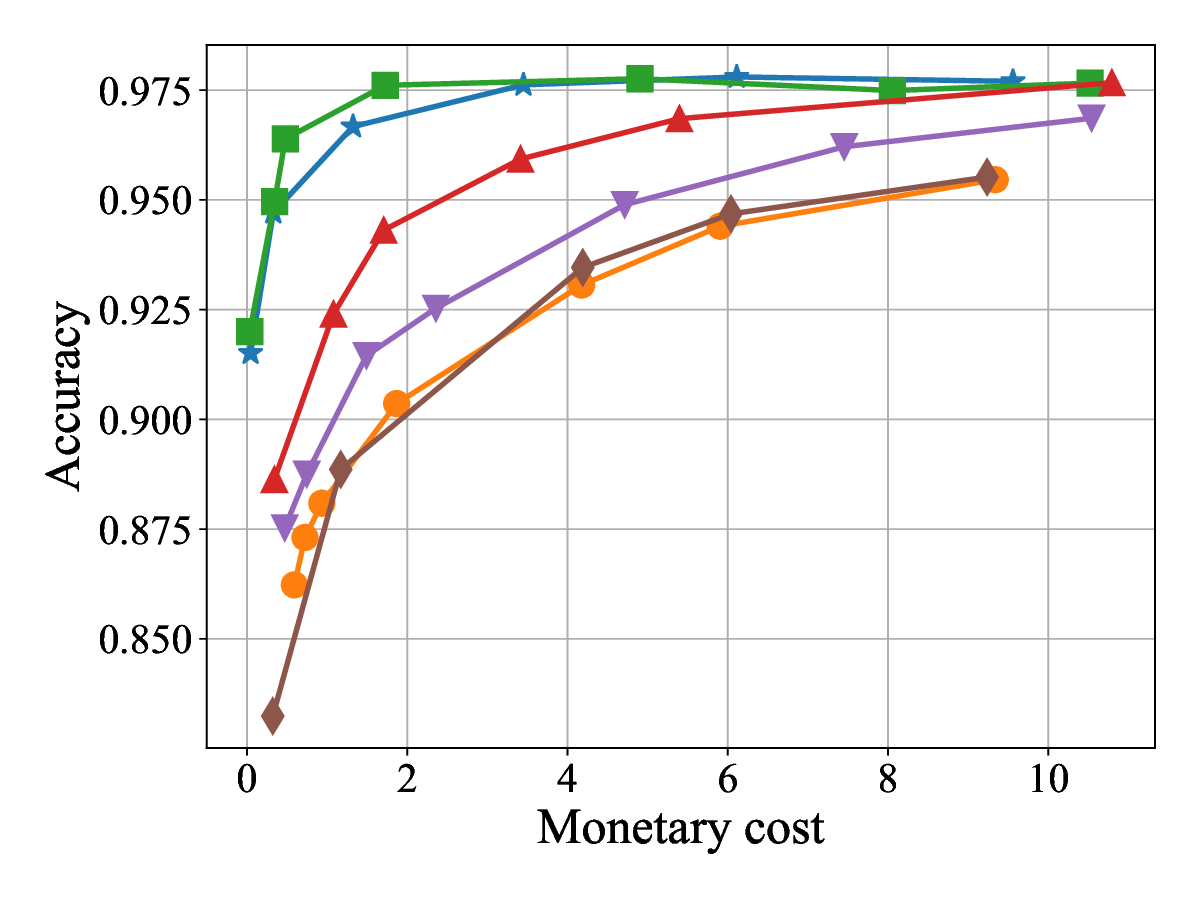}
		\caption{MNIST, Similarity: $s=100$}
		\captionsetup{justification=centering}
		\label{fig: test_loss_m_iid}
	\end{subfigure}
	\begin{subfigure}{0.33\linewidth}
		\centering
		\includegraphics[width=1\linewidth]{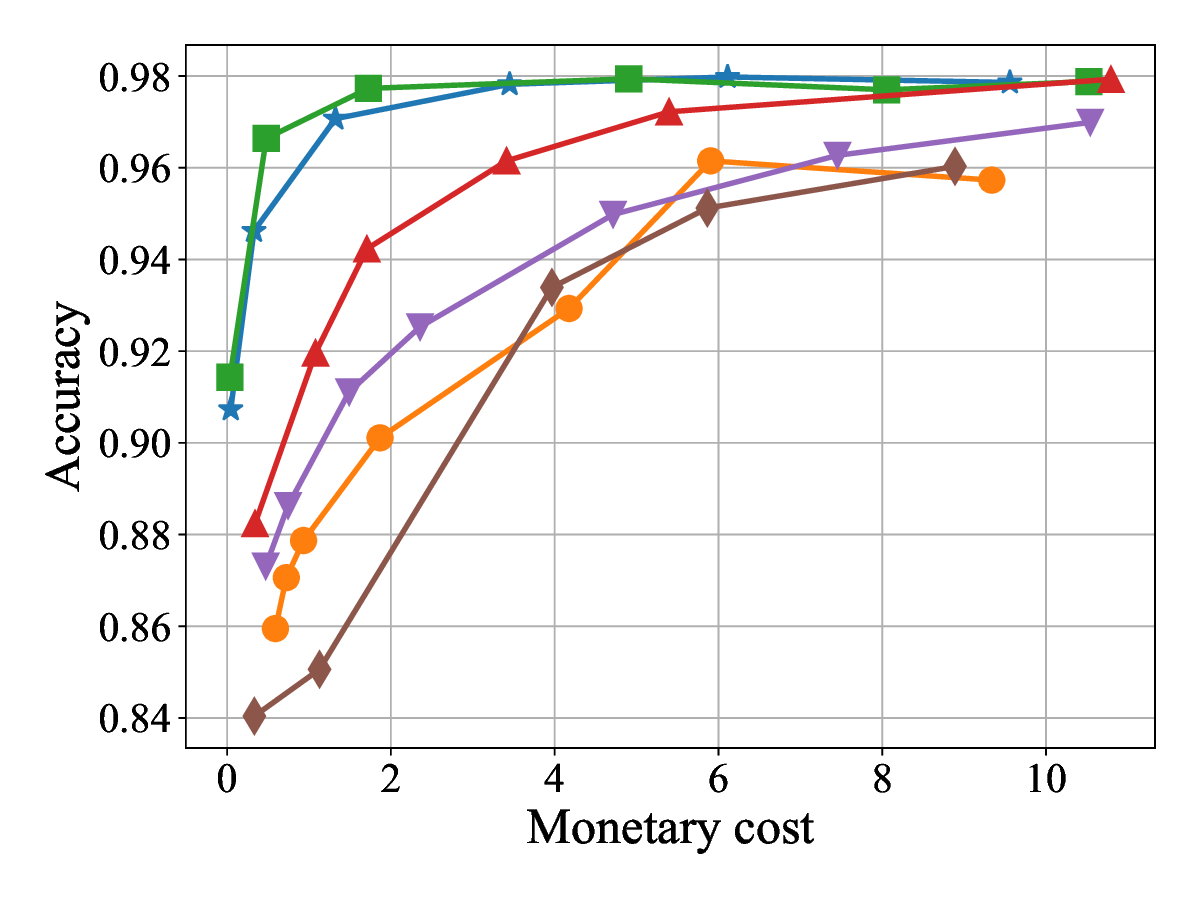}
		\caption{MNIST, Similarity: $s=70$}
		\captionsetup{justification=centering}
		\label{fig: test_loss_m_sim7}
	\end{subfigure}
	\begin{subfigure}{0.33\linewidth}
		\centering
		\includegraphics[width=1\linewidth]{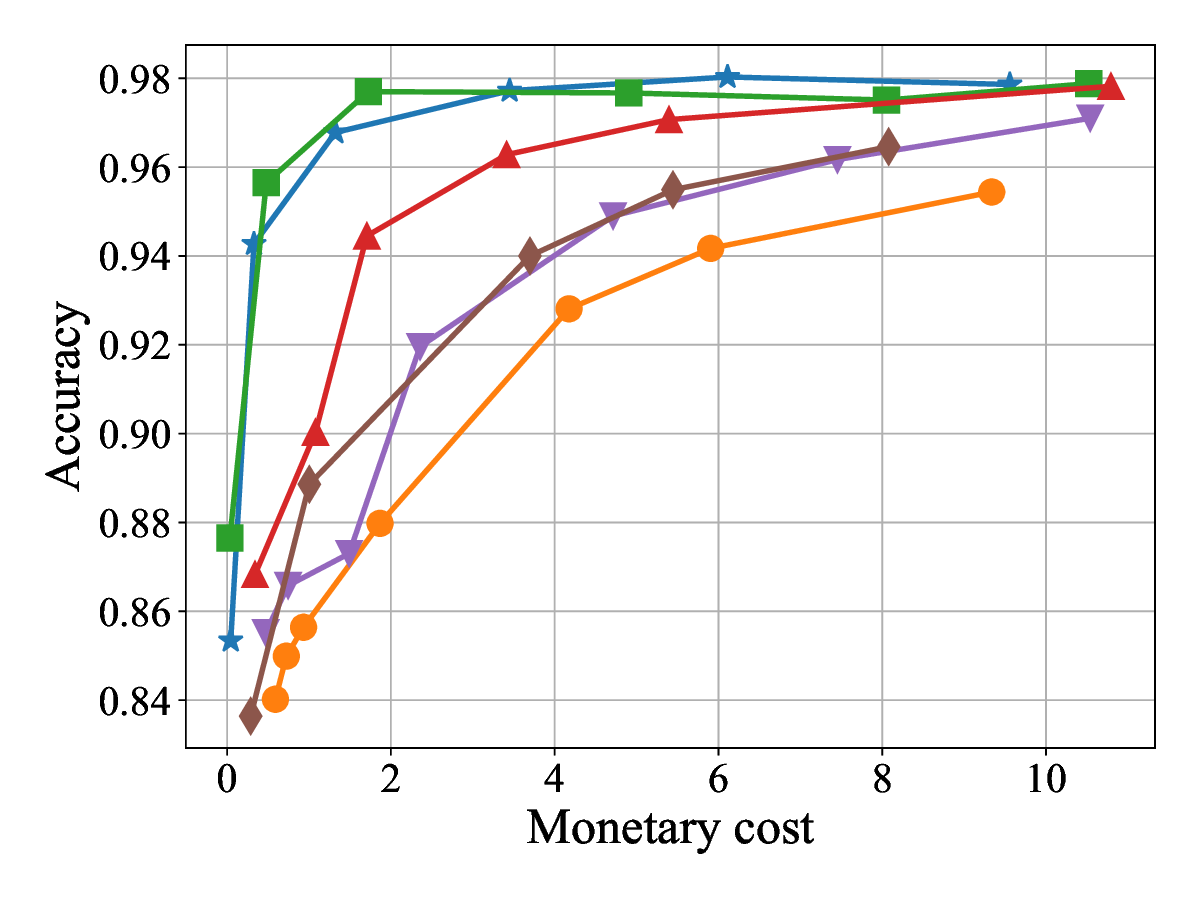}
		\caption{MNIST, Similarity: $s=30$}
		\captionsetup{justification=centering}
		\label{fig: test_loss_m_sim3}
	\end{subfigure}
	\begin{subfigure}{0.33\linewidth}
		\centering
		\includegraphics[width=1\linewidth]{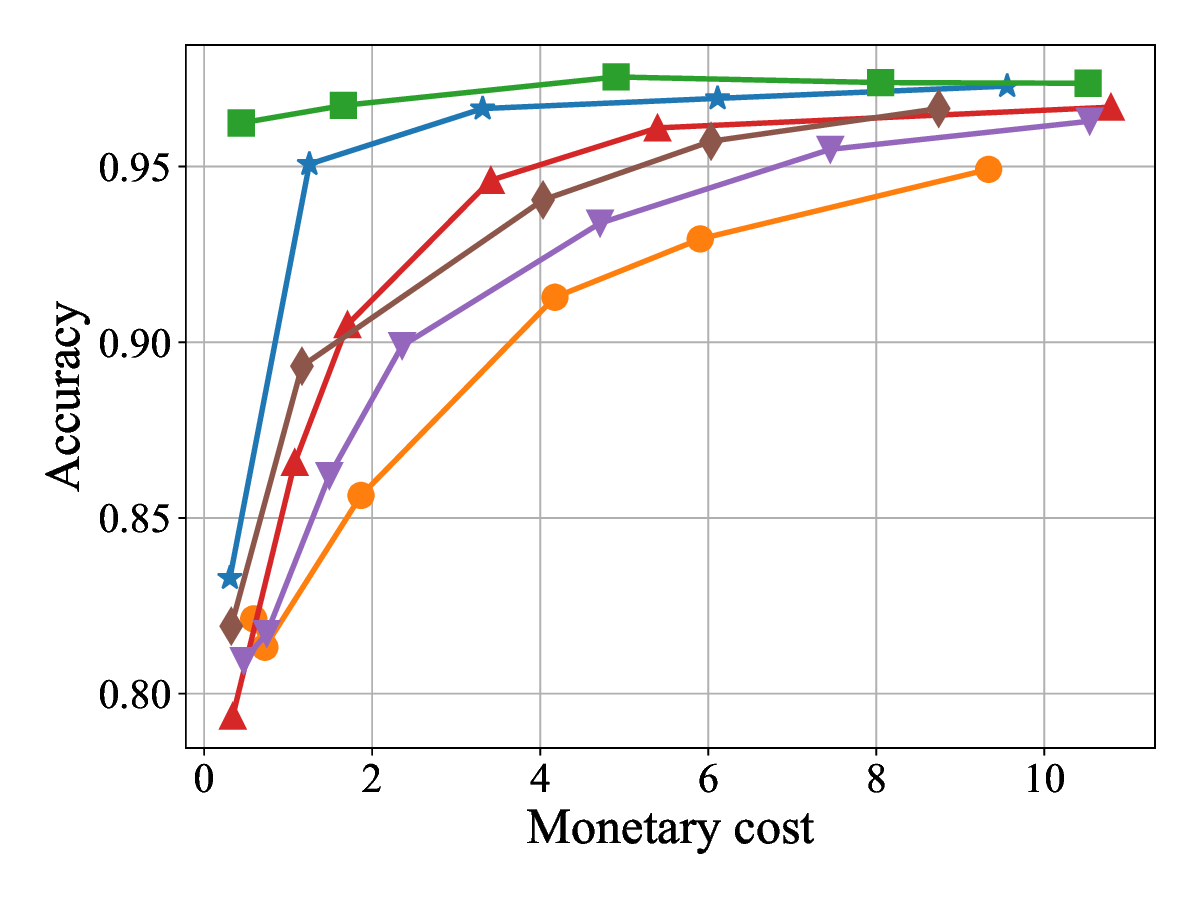}
		\caption{MNIST, Similarity: $s=0$}
		\captionsetup{justification=centering}
		\label{fig: test_loss_m_sim0}
	\end{subfigure}
	\begin{subfigure}{0.33\linewidth}
		\centering
		\includegraphics[width=0.5\linewidth]{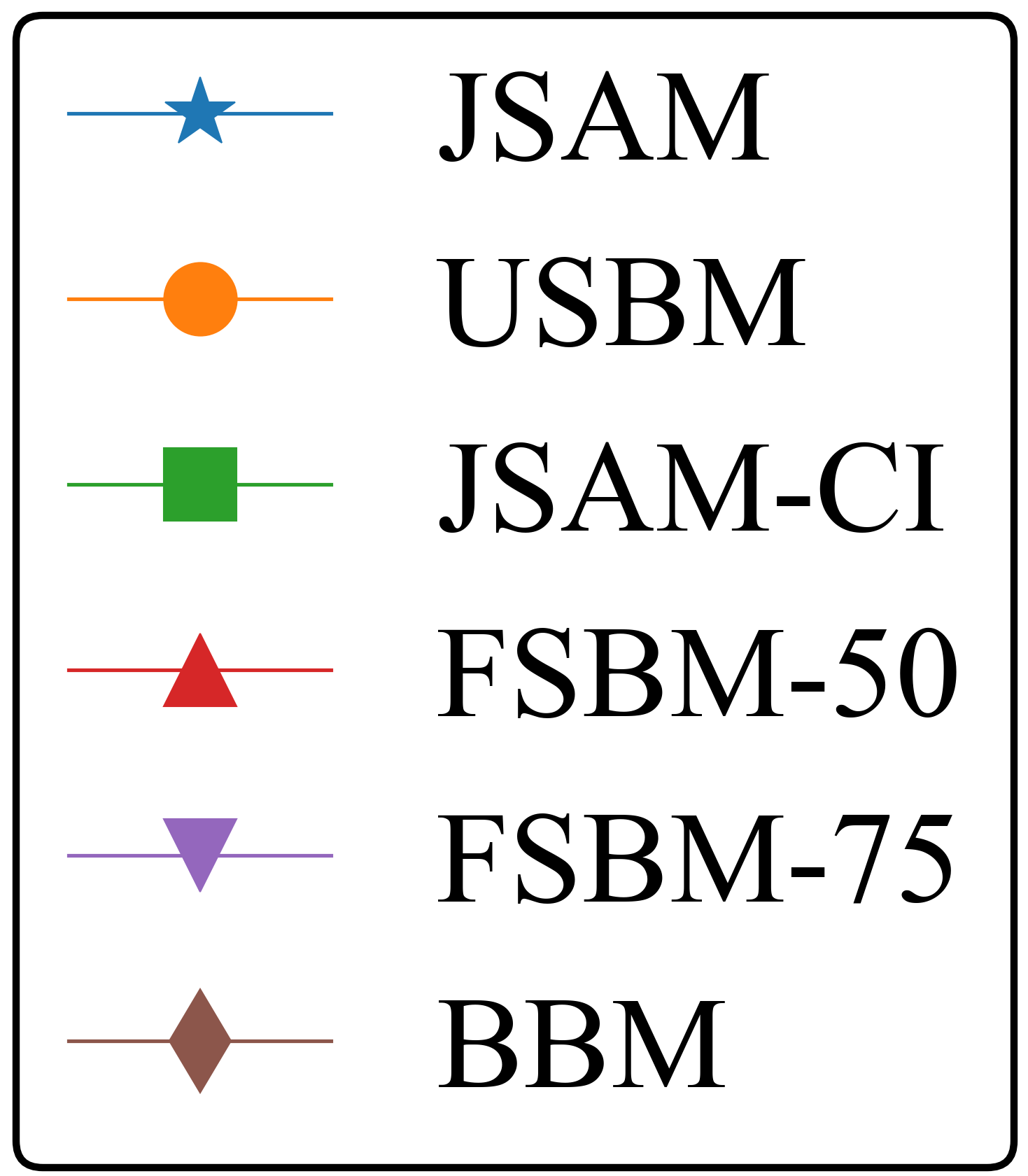}
		\vspace{38pt}
		\captionsetup{justification=centering}
		\label{fig: test_loss_tn}
	\end{subfigure}
	\caption{Test accuracy versus monetary cost.}
	\label{fig: test_acc}
\end{figure*}

\subsection{Test Loss Performance} \label{subsec: test loss}
We evaluate the test accuracy performance of JSAM and the baseline mechanisms under varying monetary costs.
The results are depicted in Figure \ref{fig: test_acc}, illustrating performance on both CIFAR-10 and MNIST datasets across four levels of data similarity ($s\% = 100\%, 70\%, 30\%, 0\%$).
We compare the test accuracy performance with small to moderate monetary costs (about 0-10) in Fig. \ref{fig: test_acc} and list the test loss performance with large monetary costs (about 500) in Tables \ref{table: test acc 1} and \ref{table: test acc 2}.

We have the following observations:
\begin{observation}[JSAM Performance]\label{ob: jsam}
	As demonstrated in Fig. \ref{fig: test_acc}(a-h), JSAM and JSAM-CI demonstrate superior test accuracies (by up to 8\%) compared to other baselines under low to moderate monetary costs.
	As monetary costs increase, shown in Tables \ref{table: test acc 1} and \ref{table: test acc 2}, JSAM's performance converges with USBM, confirming our theoretical analysis that unbiased selection becomes optimal at higher monetary costs.
\end{observation}

\begin{observation}[FSBM Characteristics]\label{ob: fsbm}	
	As shown in Fig. \ref{fig: test_acc}(a, b, c, e, f, g, h), FSBM-50 and FSBM-75 show better test accuracy than USBM (by up to 12\%) under low and moderate monetary costs by prioritizing clients with lower privacy costs.
	At higher monetary costs, shown in Tables \ref{table: test acc 1} and \ref{table: test acc 2}, FSBM's performance improves with increased client participation, highlighting the balance between monetary investment and client participation.
\end{observation}

\begin{observation}[BBM Limitations]\label{ob: bbm}
	As shown in Fig. \ref{fig: test_acc}(a, b, c, e, f, g, h), BBM generally underperforms compared to other mechanisms, primarily due to its failure to account for client privacy sensitivity heterogeneity.
	It only achieves optimal performance in cases where selected clients with high local losses happen to have low privacy sensitivities.
\end{observation}

\begin{observation}[DP Noise Impact]\label{ob: dpnosie}
	Interestingly, as shown in Fig. \ref{fig: test_acc}(a-g), JSAM occasionally outperforms JSAM-CI despite the latter's complete information advantage.
	This can be attributed to the randomness introduced by \ac{dp} noise.
 	Under conditions of high data heterogeneity and significant \ac{dp} noise, performance shows increased variability, as evident in Figure \ref{fig: test_acc}(d).
\end{observation}

\begin{table}[ht]
	\centering
	\caption{Test accuracy performance with large monetary costs (CIFAR-10).}
	\label{table: test acc 1}
	\begin{tabular}{|c|c|c|c|}
		\hline
		\textbf{Mechanism} & \textbf{Monetary Cost} & \textbf{Non-iid Degree} & \textbf{Test Accuracy} \\
		\hline
		& \multirow{4}{*}{590.5221} & $s=100$ & 77.37\%\\
		\cline{3-4}
		JSAM &  & $s=70$ & 76.13\%\\
		\cline{3-4}
		(Proposed) &  & $s=30$ & 68.20\% \\
		\cline{3-4}
		&  & $s=0$ & 50.06\% \\
		\hline
		\multirow{4}{*}{USBM} & \multirow{4}{*}{590.5437} & $s=100$ & 77.45\% \\
		\cline{3-4}
		 &  & $s=70$ & 75.93\% \\
		\cline{3-4}
		 &  & $s=30$ & 68.06\% \\
		\cline{3-4}
		 &  & $s=0$ & 48.60\% \\
		\hline
		\multirow{4}{*}{JSAM-CI} & \multirow{4}{*}{417.5662} & $s=100$ & 77.66\% \\
		\cline{3-4}
		 &  & $s=70$ & 74.96\% \\
		\cline{3-4}
		 &  & $s=30$ & 68.75\% \\
		\cline{3-4}
		 &  & $s=0$ & 44.27\% \\
		\hline
		\multirow{4}{*}{FSBM-50} & \multirow{4}{*}{341.3135} & $s=100$ & 75.35\% \\
		\cline{3-4}
		 &  & $s=70$ & 74.40\% \\
		\cline{3-4}
		 &  & $s=30$ & 65.03\% \\
		\cline{3-4}
		 &  & $s=0$ & 48.79\% \\
		\hline
		\multirow{4}{*}{FSBM-75} & \multirow{4}{*}{471.2816} & $s=100$ & 76.82\% \\
		\cline{3-4}
		 &  & $s=70$ & 75.50\% \\
		\cline{3-4}
		 &  &$s=30$ & 65.61\% \\
		\cline{3-4}
		 &  & $s=0$ & 43.37\% \\
		\hline
		\multirow{4}{*}{BBM} & 477.8747 & $s=100$ & 64.51\% \\
		\cline{2-4}
		& 468.9910 & $s=70$ & 63.20\% \\
		\cline{2-4}
		& 521.8094 &$s=30$ & 55.22\% \\
		\cline{2-4}
		& 545.6140 & $s=0$ & 49.04\% \\
		\hline
	\end{tabular}
\end{table}

\begin{table}[ht]
	\centering
	\caption{Test accuracy performance with large monetary costs (MNIST).}
	\label{table: test acc 2}
	\begin{tabular}{|c|c|c|c|}
		\hline
		\textbf{Mechanism} & \textbf{Monetary Cost} & \textbf{Non-iid Degree} & \textbf{Test Accuracy} \\
		\hline
		\multirow{4}{*} & \multirow{4}{*}{590.5221} & $s=100$ & 99.05\%\\
		\cline{3-4}
		JSAM &  & $s=70$ & 98.93\%\\
		\cline{3-4}
		(proposed) &  & $s=30$ & 98.88\% \\
		\cline{3-4}
		&  & $s=0$ & 98.13\% \\
		\hline
		\multirow{4}{*}{USBM} & \multirow{4}{*}{590.5437} & $s=100$ & 98.99\% \\
		\cline{3-4}
		&  & $s=70$ & 98.94\% \\
		\cline{3-4}
		&  & $s=30$ & 98.88\% \\
		\cline{3-4}
		&  & $s=0$ & 98.44\% \\
		\hline
		\multirow{4}{*}{JSAM-CI} & \multirow{4}{*}{417.5662} & $s=100$ & 99.04\% \\
		\cline{3-4}
		&  & $s=70$ & 98.95\% \\
		\cline{3-4}
		&  & $s=30$ & 98.92\% \\
		\cline{3-4}
		&  & $s=0$ & 97.78\% \\
		\hline
		\multirow{4}{*}{FSBM-50} & \multirow{4}{*}{341.3135} & $s=100$ & 98.91\% \\
		\cline{3-4}
		&  & $s=70$ & 98.85\% \\
		\cline{3-4}
		&  & $s=30$ & 98.75\% \\
		\cline{3-4}
		&  & $s=0$ & 97.96\% \\
		\hline
		\multirow{4}{*}{FSBM-75} & \multirow{4}{*}{471.2816} & $s=100$ & 98.97\% \\
		\cline{3-4}
		&  & $s=70$ & 98.96\% \\
		\cline{3-4}
		&  &$s=30$ & 98.76\% \\
		\cline{3-4}
		&  & $s=0$ & 97.91 \% \\
		\hline
		\multirow{4}{*}{BBM} & 478.2449 & $s=100$ & 99.06\% \\
		\cline{2-4}
		& 485.7700 & $s=70$ & 98.89\% \\
		\cline{2-4}
		& 511.2514 &$s=30$ & 97.27\% \\
		\cline{2-4}
		& 512.2826 & $s=0$ & 98.22\% \\
		\hline
	\end{tabular}
\end{table}

\subsection{Impact of the Weighted Parameter $\eta$}\label{subsec: monetary cost}
The weighted parameter $\eta$ plays a crucial role in balancing the trade-off between the selection probability adjustment and the monetary compensation.
Figure \ref{fig: weightcost} illustrates how the total monetary cost varies with different values of $\eta$ on the CIFAR-10 dataset.
 
\begin{observation}\label{ob: weighted}
	As the weighted parameter $\eta$ increases, the total monetary cost increases convexly.
\end{observation}
Observation \ref{ob: weighted} highlights that the total monetary cost not only rises with $\eta$ but does so at an accelerating rate.
This convex behavior suggests that higher $\eta$ values disproportionately favor monetary compensation over selection probability adjustments, potentially amplifying costs due to compounding effects in the compensation mechanism.
\begin{figure}
	\centering
	\includegraphics[width=0.7\linewidth]{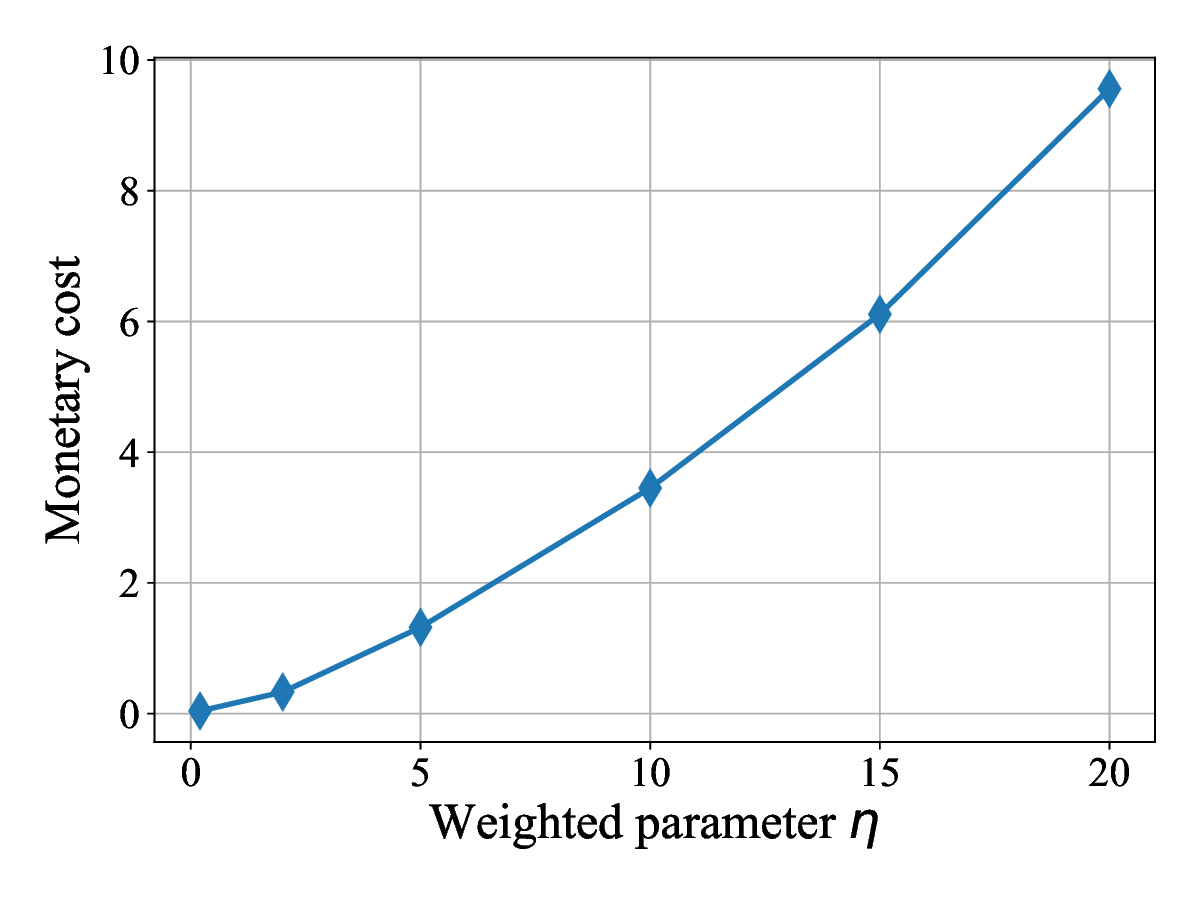}
	\caption{Monetary cost versus weighted parameter $\eta$.}
	\label{fig: weightcost}
\end{figure}

\subsection{Comparing the Number of Participating Clients}\label{subsec: num}
Figure \ref{fig: num} illustrates the relationship between total monetary cost and the number of participating clients across different mechanisms on the CIFAR-10 dataset.
\begin{observation}\label{ob: num}
	JSAM and JSAM-CI exhibit a concave increase in the number of selected clients as the monetary cost rises (\textit{e.g.}, from 27 at cost 1.32 to 41 at cost 3.45), whereas other non-adaptive baseline mechanisms maintain a constant number of participating clients regardless of the monetary budget. 
\end{observation}

Observation \ref{ob: num} reveals that JSAM adapts client participation to the budget in a concave manner, with the number of clients growing rapidly at first but at a diminishing rate as costs increase.
At low costs, it minimizes monetary cost by excluding high-cost clients, while at high costs, it maximizes data contributions by including more clients.
This flattening trend suggests a saturation effect: while JSAM leverages initial monetary budgets to expand participation, further budget increases yield smaller gains, possibly due to a limited pool of cost-effective clients or escalating per-client costs.
\begin{figure}[t]
	\centering
	\includegraphics[width=0.7\linewidth]{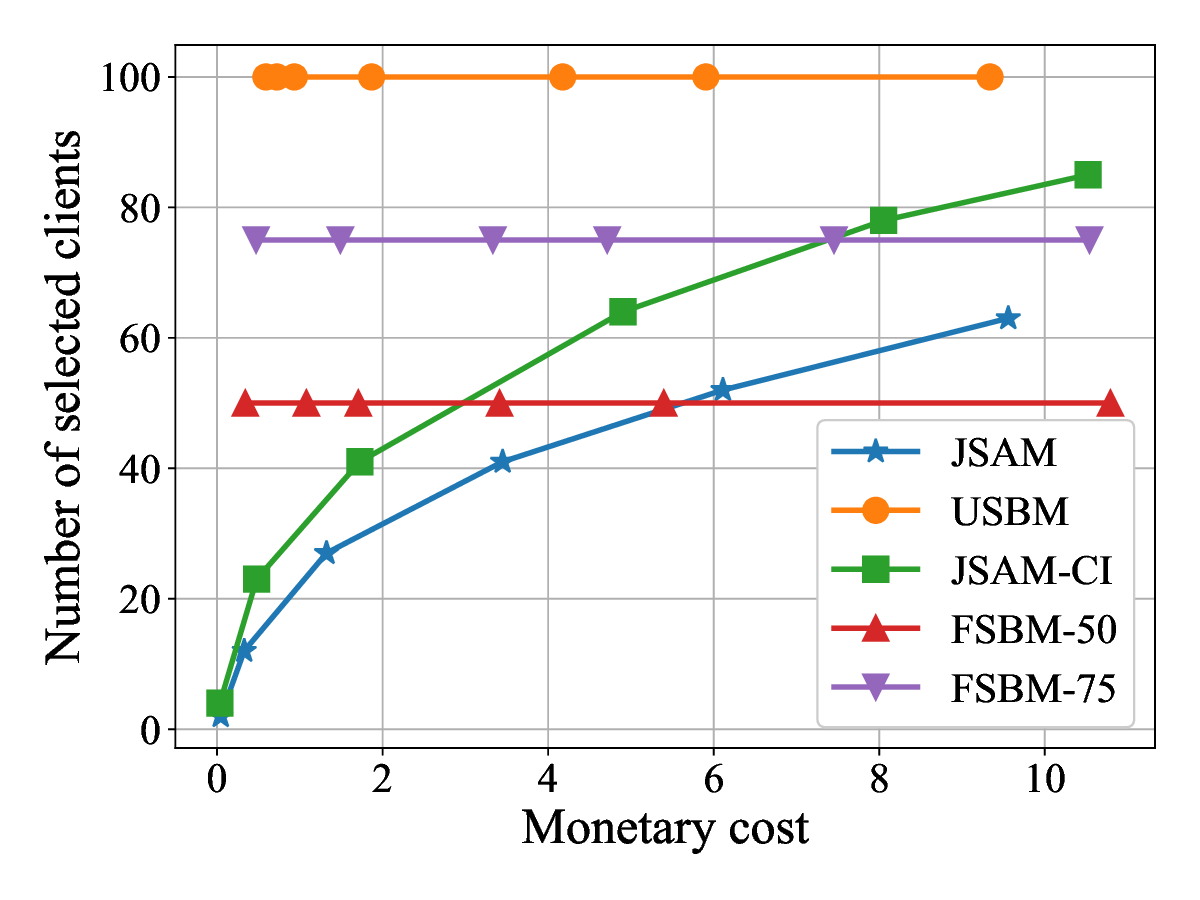}
	\caption{Number of selected clients versus monetary cost on CIFAR-10.}
	\label{fig: num}
\end{figure}
\section{Conclusion}
This paper addresses a fundamental challenge in privacy-preserving federated learning: the intricate balance between model performance, privacy protection, and monetary costs.
We begin by formulating the joint client selection and privacy compensation mechanism (JSAM) as a Bayesian-optimal mechanism design problem, capturing the inherent uncertainties and incentives in client participation.  
Our most significant finding reveals that the seemingly complex 2N-dimensional non-convex optimization problem can be reduced to just three parameters, leading to an elegant yet practical solution.
Leveraging these structural insights, we developed an efficient algorithm capable of solving the JSAM design problem effectively.
This algorithm systematically searches for the optimal threshold and corresponding selection probabilities, ensuring that the mechanism operates optimally under varying monetary constraints.
Our numerical simulations demonstrated the superior performance of JSAM compared to existing baseline mechanisms
For example, our method increases the test accuracy by up to $15\%$ when the server only wants to spend small monetary costs.

Future research will focus on relaxing JSAM's core assumptions and expanding its practical applications. Specifically, we plan to address scenarios with correlated client data distributions, moving beyond the current independent distribution assumption. We will also incorporate fairness constraints to ensure equitable treatment of participants in terms of both selection frequency and compensation rates. These extensions will help bridge the gap between theoretical mechanism design and the practical challenges faced in real-world federated learning deployments.

\bibliographystyle{IEEEtran}
\bibliography{ref}

\newpage
\appendix

\subsection{Mechanism design problem reformulation proof}\label{sec: mechanism design proof}
\subsubsection{Reformulating \ac{ic} and \ac{ir} constraints}\label{subsubsec: reformulate ic ir}
For simplicity in notation, let us first define the expected payment and expected allocation for a client $k$ with cost $c_k$ as:
\begin{align}
	&\pi_k(c_k) = \mathbb{E}_{\boldsymbol{c}_{-k}}\left[\pi_k(c_k,\boldsymbol{c}_{-k})\right],\\
	&\epsilon_{k}(c_k) = \mathbb{E}_{\boldsymbol{c}_{-k}}\left[\epsilon_k(c_k,\boldsymbol{c}_{-k})\right].
\end{align}
First, we prove the following lemma that provides a sufficient and necessary condition for the \ac{ic} constraint.

\begin{lemma}\label{lemma: ic}
	For any client $k \in \mathcal{N}$, the payment function $\pi_k$ satisfies the \ac{ic} constraint if and only if:
	\begin{enumerate}
		\item The expected privacy budget $\epsilon_k(c_k)$ is non-increasing in $c_k$.
		\item The payment function is of the form:
		\begin{align}\label{equ: payment ic}
			\pi_k(c_k) = \pi_k(0) + c_k\epsilon_k(c_k) - \int_{0}^{c_k}\epsilon_k(z)dz.
		\end{align}
	\end{enumerate}
\end{lemma}
\begin{proof}
	Let $U_k(\hat{c}_k; c_k) = \pi_k(\hat{c}_k) - c_k\epsilon_k(\hat{c}_k)$ be the utility of an agent with true cost $c_k$ who reports a cost of $\hat{c}_k$. The \ac{ic} constraint requires that for all $c_k$ and $\hat{c}_k$, $U_k(c_k; c_k) \ge U_k(\hat{c}_k; c_k)$. Let $U_k(c_k) \equiv U_k(c_k; c_k)$ be the equilibrium utility.

\paragraph{Necessity (IC $\implies$ Conditions 1 and 2)}
Assuming the mechanism is \ac{ic}, the agent's utility $U_k(\hat{c}_k; c_k)$ is maximized at $\hat{c}_k=c_k$. By the envelope theorem, we have:
\begin{equation}
\begin{aligned}
	\frac{d U_k(c_k)}{d c_k} =& \left. \frac{\partial U_k(\hat{c}_k; c_k)}{\partial c_k} \right|_{\hat{c}_k=c_k}\\
	=& \left. \frac{\partial}{\partial c_k} \left( \pi_k(\hat{c}_k) - c_k\epsilon_k(\hat{c}_k) \right) \right|_{\hat{c}_k=c_k}\\
	=& -\epsilon_k(c_k).
\end{aligned}
\end{equation}
Integrating both sides from $0$ to $c_k$ yields:
\begin{equation}
\begin{aligned}
	&U_k(c_k) - U_k(0) = -\int_{0}^{c_k}\epsilon_k(z)dz\\
	\implies& U_k(c_k) = U_k(0) - \int_{0}^{c_k}\epsilon_k(z)dz.
\end{aligned}
\end{equation}
By substituting the definitions $U_k(c_k) = \pi_k(c_k) - c_k\epsilon_k(c_k)$ and $U_k(0) = \pi_k(0)$, we obtain the payment rule in \eqref{equ: payment ic}.

To prove that $\epsilon_k(c_k)$ is non-increasing, consider two costs $c_k$ and $c'_k$. The \ac{ic} constraint implies:
\begin{align}\label{equ: ic_proof_1}
	\pi_k(c_k) - c_k\epsilon_k(c_k) &\ge \pi_k(c'_k) - c_k\epsilon_k(c'_k), \\ \label{equ: ic_proof_2}
	\pi_k(c'_k) - c'_k\epsilon_k(c'_k) &\ge \pi_k(c_k) - c'_k\epsilon_k(c_k).
\end{align}
Adding \eqref{equ: ic_proof_1} and \eqref{equ: ic_proof_2} and rearranging gives:
\begin{align}
	c_k\epsilon_k(c_k)-c_k\epsilon_k(c'_k)&\le c'_k\epsilon_k(c_k)-c'_k\epsilon_k(c'_k)\\
	(c_k - c'_k) \left[\epsilon_k(c_k) - \epsilon_k(c'_k)\right] &\le 0.
\end{align}
This inequality implies that if $c_k > c'_k$, we must have $\epsilon_k(c_k) \le \epsilon_k(c'_k)$. Thus, $\epsilon_k(c_k)$ is non-increasing in $c_k$.

\paragraph{Sufficiency (Conditions 1 and 2 $\implies$ IC)}

By the form of the payment, we have 
\begin{equation}\label{equ: ic ne 1}
\begin{aligned}
	&\left[\pi_k(x) - c_k\epsilon_k(x)\right] - \left[\pi_k(c_k) - c_k\epsilon_k(c_k)\right] \\
	=&\,\, x\epsilon_k(x) - \int_{0}^{x}\epsilon_k(z)dz - c_k\epsilon_k(x) + \int_{0}^{c_k}\epsilon_k(z)dz\\
	=&\,\, (x-c_k) \epsilon_k(x) - \int_{c_k}^{x}\epsilon_{k}(z)dz.
\end{aligned}
\end{equation}
Combining the condition that $\epsilon_k(c_k)$ is weakly decreasing with $c_k$ with (\ref{equ: ic ne 1}), we have 
\begin{align}
	(x-c_k)\cdot\epsilon_k(x) - \int_{c_k}^{x}\epsilon_{k}(z)dz \le 0.
\end{align}
Thus, the \ac{ic} constraint holds.
This completes the proof.
\end{proof}

Next, we incorporate the \ac{ir} constraint to derive the final payment structure in Lemma \ref{lemma: ir ic}.

\begin{proof}
Following the payment function in Lemma \ref{lemma: ic}, the \ac{ir} condition implies
\begin{align}\label{equ: ir}
	\pi_k(c_k) - c_k\epsilon_k(c_k) = \pi_k(0)- \int_{0}^{c_k}\epsilon_{k}(z)dz \ge 0,
\end{align}
for all $k \in \mathcal{N}$.

This means the \ac{ir} constraint is most difficult to satisfy for the agent with the highest possible cost. 
To ensure the constraint holds for all $c_k \in \mathbb{R}^+$, we only need to verify it for the worst case, as $c_k \to \infty$:
Thus, we focus on the case that $c_k = \infty$, 
\begin{align}
	\pi_k(0)- \int_{0}^{\infty}\epsilon_{k}(z)dz \ge 0, \forall k \in \mathcal{N}.
\end{align}
Substituting it back to (\ref{equ: payment ic}) yields that the payment structure for any $k \in \mathcal{N}$,
\begin{equation}
\begin{aligned}
	\pi_k(c_k) &= \int_{0}^{\infty}\epsilon_{k}(z)dz + c_k\epsilon_{k}(c_k) - \int_{0}^{c_k}\epsilon_{k}(z)dz + d_k\\
	& = \int_{c_k}^{\infty}\epsilon_{k}(z)dz + c_k\epsilon_{k}(c_k) + d_k,
\end{aligned}
\end{equation}
where $d_k$ is non-negative.
\end{proof}

\subsubsection{Proof of Theorem \ref{theorem: reformulation}}
\begin{proof}
Let $f_k(\cdot)$ and $F_k(\cdot)$ be the probability density and cumulative distribution functions of client $k$'s cost $c_k$, respectively.

From Lemma \ref{lemma: ir ic}, the expected payment for agent $k$ that satisfies \ac{ic} and \ac{ir} is given by:
\begin{equation}\nonumber
\begin{aligned}
	&\mathbb{E}_{c_k}[\pi_k(c_k)] \\
	=&\,\,\mathbb{E}_{c_k}\left[\left( \int_{c_k}^{\infty}\epsilon_{k}(z)dz + c_k\epsilon_{k}(c_k)\right)\right] + d_k\\
	=& \,\,\int_{\boldsymbol{x}_{-k}}\!\int_{0}^{\infty}\!\int_{0}^{z}\! \epsilon_{k}(z,\boldsymbol{x}_{-k})f_{-k}(\boldsymbol{x}_{-k})f_k(x_k)dx_kdzd\boldsymbol{x}_{-k}\!\\
	&\,\,+\! \int_{\boldsymbol{x}_{-k}}\!\int_{0}^{\infty}\! x_k\epsilon_{k}(x_k,\boldsymbol{x}_{-k})f_{-k}(\boldsymbol{x}_{-k})f_k(x_k)dx_kd\boldsymbol{x}_{-k} \!+\! d_k\\
	=&\,\, \int_{\boldsymbol{x}_{-k}}\!\int_{0}^{\infty}\!\epsilon_{k}(z,\boldsymbol{x}_{-k})(F_k(z)-F_k(0))f_{-k}(\boldsymbol{x}_{-k})dzd\boldsymbol{x}_{-k}\!\\
	&\,\,+\! \int_{\boldsymbol{x}_{-k}}\int_{0}^{\infty}\!\! x_k\epsilon_{k}(x_k,\boldsymbol{x}_{-k})f_{-k}(\boldsymbol{x}_{-k})f_k(x_k)dx_kd\boldsymbol{x}_{-k} \!+\! d_k\\
	=&\,\, \int_{\boldsymbol{x}}\!\!\epsilon_{k}(x_k,\boldsymbol{x}_{-k})\!\!\left[x_k\!+\!\frac{F_k(x_k)-F_k(0)}{f_k(x_k)}\right]\!f_k(x_k)f_{-k} (\boldsymbol{x}_{-k})d\boldsymbol{x} \\
	&\,\,+ d_k.
\end{aligned}
\end{equation}

For any distribution of the costs $\boldsymbol{c}$, we have 
\begin{equation}\label{equ: payment2}
\begin{aligned}
	&\mathbb{E}_{\boldsymbol{c}} [\pi_k(c_k)] = d_k\\
	&\!+\!\mathbb{E}_{\boldsymbol{c}}\!\left[\epsilon_{k}(c_k,\boldsymbol{c}_{-k})\!\!\left[c_k\!+\!\frac{F_k(c_k)\!-\!F_k(0)}{f_k(c_k)}\!\right]\!f_k(c_k)f_{-k} (\boldsymbol{c}_{-k})d\boldsymbol{x}\right].
\end{aligned}
\end{equation}
Substituting (\ref{equ: payment2}) into Problem \ref{prob: opt}, we have
\begin{equation}\label{opt: reformulate 1}
	\begin{aligned}\nonumber
		\min_{\begin{array}{c}\scriptstyle 			
			\boldsymbol{\epsilon}(\boldsymbol{c}),\\
			[-4pt]\scriptstyle 	 \boldsymbol{p}^\text{s}(\boldsymbol{c})\end{array}}&\!\!\mathbb{E}_{\boldsymbol{c}}\!\!\left[\!\eta\!\!\left(\!\!\left\| \boldsymbol{p}^\text{s}(\boldsymbol{c})\!-\!\boldsymbol{p}^\text{u}\right\|_1\!\! +\!\! \sqrt{\left\| \boldsymbol{p}^\text{s}(\boldsymbol{c})\!-\!\boldsymbol{p}^\text{u}\right\|_1^2\! +\! Q\! \sum_{k=1}^{N}\! \frac{(p^\text{s}_{k}(\boldsymbol{c}))^2}{\epsilon_k^2(\boldsymbol{c})}}\!\right)\!\right] \\
		&+ \mathbb{E}_{\boldsymbol{c}}\left[\sum_{k=1}^{N}\left(\epsilon_{k}(\boldsymbol{c})\left(c_k+\frac{F_k(c_k)}{f_k(c_k)} \right) + d_k\right)\right]\\
		\text{ }\text{s.t.}& \sum_{k=1}^{N}p^\text{s}_{k}(\boldsymbol{c}) = 1,p^\text{s}_{k}, \epsilon_k,d_k \ge 0, \forall k \in \mathcal{N},\\
		&\epsilon_{k} \text{ is weakly decreasing with }c_k, \forall k \in \mathcal{N}.
	\end{aligned}
\end{equation}

Obviously, at the optimal solution, we always have $d_k = 0$, for all $k \in \mathcal{N}$.
Setting $d_k = 0$ for all $k \in \mathcal{N}$ and considering the point-wise optimization form yields the reformulated optimization problem, \textit{i.e.},
\begin{equation}
	\begin{aligned}
		\min_{\boldsymbol{\epsilon}, \boldsymbol{p}_\text{s}}&\text{ } \eta\left( \left\| \boldsymbol{p}^\text{s}-\boldsymbol{p}^\text{u}\right\|_1 + \sqrt{\left\| \boldsymbol{p}^\text{s}-\boldsymbol{p}^\text{u}\right\|_1^2 + Q \sum_{k=1}^{N} \frac{(p^\text{s}_{k})^2}{\epsilon_k^2}}\right)\\ &+\sum_{k=1}^{N}\epsilon_{k}\left(c_k+\frac{F_k(c_k)}{f_k(c_k)} \right)\\
		\text{ }\text{s.t.}& \sum_{k=1}^{N}p^\text{s}_{k} = 1, p^\text{s}_{k} \ge 0, \forall k\in \mathcal{N},\\
		& \epsilon_{k} \text{ is weakly decreasing with }c_k,\forall k\in \mathcal{N}.
	\end{aligned}
\end{equation}
This completes the proof.
\end{proof}

\subsection{Proof of Theorem \ref{theorem: structure}}
We prove Theorem \ref{theorem: structure} by establishing three lemmas: Lemma \ref{lemma: S+}, Lemma \ref{lemma: S-}, and Lemma \ref{lemma: threshold}.
We prove them respectively.
\begin{lemma}\label{lemma: S+}
	At the optimal solution of Problem \ref{prob: relaxed}, there is $\underline{\text{at most one}}$ client in the set $\mathcal{S}^{+*}$.
\end{lemma}
\begin{proof}
	We prove Lemma \ref{lemma: S+} by contradiction. Suppose for the sake of contradiction that the optimal set $\mathcal{S}^{+*}$ contains more than one client. Let clients 1 and 2 be two distinct clients in this set, with their allocations in the optimal solution being $(p^*_1, \epsilon^*_1)$ and $(p^*_2, \epsilon^*_2)$, respectively. Without loss of generality, assume their virtual costs satisfy $v_1 < v_2$.
	
	Now, consider reallocating the total probability $P = p^*_1 + p^*_2$ between only these two clients, while keeping all other allocations fixed. For any new probabilities $p_1, p_2$ such that $p_1 + p_2 = P$ and $p_1, p_2 \ge 1/N$, the $L_1$-norm term $\|\boldsymbol{p}^\textnormal{s} - \boldsymbol{p}^\textnormal{u}\|_1$ of the objective function remains constant. Therefore, for the original allocation to be optimal, it must also be an optimal solution to the subproblem of minimizing the remaining objective terms by allocating the fixed resources between just clients 1 and 2. This subproblem is formulated in \eqref{opt: 2-client +} as follows:
	\begin{equation}\label{opt: 2-client +}
		\begin{aligned}
			\min& \frac{p_1^2}{\epsilon_1^2} + \frac{p_2^2}{\epsilon_2^2},\\
			\textnormal{Var.}& \text{ }\epsilon_1, p_1, \epsilon_2, p_2,\\
			\textnormal{s.t.}& \text{ } p_1 + p_2 = p^*_1 + p^*_2 = \tilde{P},\\
			& v_1\epsilon_1 + v_2\epsilon_2 = B_1^* + B_2^* = \tilde{B},\\
			&p_1,p_2 \ge 1/N,\\
			&\epsilon_1,\epsilon_2 > 0.
		\end{aligned}
	\end{equation}
	
	We characterize the optimal solution of (\ref{opt: 2-client +}) using the KKT conditions.
	The Lagrangian is given by:
	\begin{equation}
		\begin{aligned}
			L_\ell =& \sum_{k=1}^{2} \frac{p_k^2}{\epsilon_k^2} + \lambda_1 (p_1 + p_2 - \tilde{P}) + \lambda_2 (v_1\epsilon_1 + v_2\epsilon_2 - \tilde{B})\\
			&+ \mu^p_1 (1/N-p_1) + \mu^p_2(1/N-p_2) - \mu^e_1 \epsilon_1 - \mu^e_2 \epsilon_2.
		\end{aligned}
	\end{equation}
	The KKT conditions of Problem (\ref{opt: 2-client +}) are
	
	\begin{align}
		\label{global kkt: 1}
		&\frac{\partial L_\ell}{\partial p_k} = \frac{2p_k}{\epsilon_k^2} + \lambda_1 - \mu_k^p=0, \forall k \in \{1,2\},\\	\label{global kkt: 2}
		& \frac{\partial L_\ell}{\partial \epsilon_k} = \frac{-2p_k^2}{\epsilon_k^3} + \lambda_2v_k=0,\forall k \in \{1,2\},\\	\label{global kkt: 3}
		& \mu^p_1 (1/N-p_1) = 0,\\	\label{global kkt: 4}
		& \mu^p_2 (1/N-p_2) = 0,\\
		\label{global kkt: 5}
		& p_1 + p_2  = \tilde{P},\\
		\label{global kkt: 6}
		& v_1\epsilon_1 + v_2\epsilon_2 = \tilde{B},\\
		&\mu^p_1,\mu^p_2\ge 0.
	\end{align}
We derive a contradiction by analyzing the following two cases.
	\paragraph{Case 1: $p_1, p_2 > 1/N$}
	In this scenario, the inequality constraints on the probability allocations are inactive. 
	The KKT conditions therefore imply that their corresponding dual variables are zero, i.e., $\mu^p_1 = \mu^p_2 = 0$. 
	Substituting this result into the first-order optimality conditions from \eqref{global kkt: 1} and \eqref{global kkt: 2} yields the following relations for $k \in \{1, 2\}$:
	\begin{align}\label{global kkt: sol1}
		&\frac{p_k}{\epsilon_k^2} = -\lambda_1/2 = c_3,\\\label{global kkt: sol2}
		&\frac{p_k^2}{\epsilon_k^3} = \frac{\lambda_2}{2}v_k = c_4v_k,
	\end{align}
	for some constants $c_3, c_4$.
	Substituting (\ref{global kkt: sol1}) to (\ref{global kkt: sol2}), we obtain 
	\begin{align}
		\frac{p_k^2}{\epsilon_k^3} =&\; c_3\frac{p_k}{\epsilon_k} = c_4v_k,\\\label{global kkt: sol3}
		p_k =&\; \frac{c_4}{c_3}v_k\epsilon_k,
	\end{align}
	for both $k\in\{1,2\}$.
	Using the constraints $p_1+p_2 = \tilde{P}$ and $v_1\epsilon_1 + v_2\epsilon_2 = \tilde{B}$, we have
	\begin{align}\label{global kkt: sol4}
		\frac{c_4}{c_3} = \frac{\tilde{P}}{\tilde{B}}.
	\end{align}
	Substituting (\ref{global kkt: sol3}) and (\ref{global kkt: sol4}) into (\ref{global kkt: sol1}) and (\ref{global kkt: sol2}), and solving the equations (\ref{global kkt: sol1}) and (\ref{global kkt: sol2}) yields
	\begin{align}
		c_3 =& \sum_{k} \frac{v_k^2\tilde{P}}{\tilde{B}^2},	c_4 = \sum_{k} \frac{v_k^2\tilde{P}^2}{\tilde{B}^3}.
	\end{align}
	Thus, the optimal solution of Problem (\ref{opt: 2-client +}) in case 1 is
	\begin{align}
		&p_k = \frac{v_k^2}{\sum_{k}v_k^2}\tilde{P}, \forall k \in\{1,2\},\\
		&\epsilon_k = \frac{v_k}{\sum_{k}v_k^2}\tilde{B}, \forall k \in\{1,2\}.
	\end{align}
	The corresponding objective of Problem (\ref{opt: 2-client +}) in case 1 is 
	\begin{align}
		f^\textnormal{case1} = c_4\tilde{B} = \sum_{k} \frac{v_k^2\tilde{P}^2}{\tilde{B}^3}.
	\end{align}
	
	\paragraph{Case 2: $p_1 = p_{1,r} > 1/N$ and $p_2 = p_{2,r} = 1/N$}
	By the KKT conditions (\ref{global kkt: 2}) and (\ref{global kkt: 5}),  We have
	\begin{align}
		&\frac{p_k^2}{\epsilon_k^3} = \frac{\lambda_2}{2}v_k = c_4'v_k,\\
		& p_1 = p_{r,1} = \tilde{P} - 1/N,\\
		& p_2 = p_{r,2} = 1/N,
	\end{align}
	for a constant $c_4'$.
	Solving the above equations yields
	\begin{align}
		c_4' =& \left(\frac{\sum_{k}(p_{r,k}^2v_k^2)^{1/3}}{\tilde{B}}\right)^3.
	\end{align}
	Thus, the corresponding objective of Problem (\ref{opt: 2-client +}) in case 2 is
	\begin{align}
		f^\textnormal{case2} = c_4'\tilde{B} = \frac{(\sum_{k}(p_{r,k}^2v_k^2)^{1/3})^3}{\tilde{B}^2}.
	\end{align}
	We then the compare the objectives $f^\textnormal{case1}$ and $f^\textnormal{case2}$ to show that reallocation (case 2) can achieve better results than case 1.
	The objectives of Problem (\ref{opt: 2-client +}) in cases 1 and 2 are
	\begin{equation}
		\begin{aligned}
			f^\textnormal{case1}\tilde{B}^2 =&\; (p_1+p_{2})^2(v_1^2 + v_2^2)\\
			=&\; (p_{r,1}+p_{r,2})^2(v_1^2 + v_2^2)\\
			=& \;(p_{r,1}^2 + 2p_{r,1}p_{r,2} + p_{r,2}^2)(v_1^2 + v_2^2)\\
			=&\;p_{r,1}^2v_1^2 + p_{r,2}^2v_2^2\\
			&\;+  p_{r,1}p_{r,2}v_1v_2\underbrace{\left(\oldfrac{p_{r,1}v_2}{p_{r,2}v_1}+2\oldfrac{v_1}{v_2} + 2\oldfrac{v_2}{v_1} + \oldfrac{p_{r,2}v_1}{p_{r,1}v_2} \right)}_{\Delta_1},
		\end{aligned}
	\end{equation}
and
	\begin{equation}
		\begin{aligned}
			&f^\textnormal{case2}\tilde{B}^2\\
			=&\; p_{r,2}^2v_2^2 + p_{r,1}^2v_1^2 + 3p_{r,2}^{4/3}v_2^{4/3}p_{r,1}^{2/3}v_1^{2/3}+ 3p_{r,2}^{2/3}v_2^{2/3}p_{r,1}^{1/3}v_1^{1/3}\\
			=&\;p_{r,1}^2v_1^2 + p_{r,2}^2v_2^2 + p_{r,1}p_{r,2}v_1v_2\underbrace{\left(3 \oldfrac{p_{r,2}^{1/3}v_2^{1/3}}{p_{r,1}^{1/3}v_1^{1/3}} + 3 \oldfrac{p_{r,1}^{1/3}v_1^{1/3}}{p_{r,2}^{1/3}v_2^{1/3}}\right)}_{\Delta_2}.
		\end{aligned}
	\end{equation}
	Because the terms except $\Delta_1$ and $\Delta_2$ are the same in $f^\textnormal{case1}\tilde{B}^2$ and $f^\textnormal{case2}\tilde{B}^2$, the key is to compare $\Delta_1$ and $\Delta_2$.
	
	For $\Delta_2$, by the fact that $f(x) = x + 1/x $ increases with $x(x\ge1)$, we have 
	\begin{equation}
		\begin{aligned}
			&\Delta_2 = 3\left( \oldfrac{p_{r,2}^{1/3}v_2^{1/3}}{p_{r,1}^{1/3}v_1^{1/3}} + \oldfrac{p_{r,1}^{1/3}v_1^{1/3}}{p_{r,2}^{1/3}v_2^{1/3}}\right)= 3f\left(\max\{\oldfrac{p_{r,2}v_2}{p_{r,1}v_1}\}^{1/3}\right)
			\\\le&\;
			2\sqrt{2}f(\max\{\oldfrac{p_{r,2}v_2}{p_{r,1}v_1}\}^{1/2}) +\max(3f(x^{1/3})- 2\sqrt{2}f(x^{1/2}))\\
			=&\; 2\sqrt{2}\left( \oldfrac{p_{r,2}^{1/2}v_2^{1/2}}{p_{r,1}^{1/2}v_1^{1/2}} +  \oldfrac{p_{r,1}^{1/2}v_1^{1/2}}{p_{r,2}^{1/2}v_2^{1/2}}\right) + 6 - 4\sqrt{2}.
		\end{aligned}
	\end{equation}
	
	For $\Delta_1$, note that $\frac{p_{r,1}}{p_{r,2}, \frac{v_{2}}{v_{1}}} > 1$.
	Then by the fact that $f(x) = x + 1/x$ is increasing with $x(x\ge1)$, we have
	\begin{equation}
		\begin{aligned}
			\Delta_1 =&\; \left(\oldfrac{p_{r,1}v_2}{p_{r,2}v_1}+2\oldfrac{v_1}{v_2} + 2\oldfrac{v_2}{v_1} + \oldfrac{p_{r,2}v_1}{p_{r,1}v_2} \right)\\
			>&\; \left(\oldfrac{p_{r,1}}{p_{r,2}}+2\oldfrac{v_1}{v_2} + 2\oldfrac{v_2}{v_1} + \oldfrac{p_{r,2}}{p_{r,1}} \right).
		\end{aligned}
	\end{equation}
	Therefore, the difference between $\Delta_1$ and $\Delta_2$ is given by
	\begin{equation}\label{equ:delta}
		\begin{aligned}
			&\Delta_1 - \Delta_2\\
			\ge& \left(\oldfrac{p_{r,1}v_2}{p_{r,2}v_1}+2\oldfrac{v_1}{v_2} + 2\oldfrac{v_2}{v_1} + \oldfrac{p_{r,2}v_1}{p_{r,1}v_2} \right)\\
			&- 2\sqrt{2}\left( \oldfrac{p_{r,2}^{1/2}v_2^{1/2}}{p_{r,1}^{1/2}v_1^{1/2}} +  \oldfrac{p_{r,1}^{1/2}v_1^{1/2}}{p_{r,2}^{1/2}v_2^{1/2}}\right) - (6 - 4\sqrt{2})\\
			>& \left(\oldfrac{p_{r,1}}{p_{r,2}}+2\oldfrac{v_1}{v_2} + 2\oldfrac{v_2}{v_1} + \oldfrac{p_{r,2}}{p_{r,1}} \right)\\
			&- 2\sqrt{2}\left( \oldfrac{p_{r,2}^{1/2}v_2^{1/2}}{p_{r,1}^{1/2}v_1^{1/2}} +  \oldfrac{p_{r,1}^{1/2}v_1^{1/2}}{p_{r,2}^{1/2}v_2^{1/2}}\right) - (6 - 4\sqrt{2})\\
			=& \underbrace{\left(\sqrt{\oldfrac{p_{r,1}}{p_{r,2}}} - \sqrt{2}\sqrt{\frac{v_1}{v_2}}\right)^2\! \!+\! \left(\sqrt{\oldfrac{p_{r,2}}{p_{r,1}}} \!-\! \sqrt{2}\sqrt{\frac{v_2}{v_1}}\right)^2}_{A} \!- (6 \!-\! 4\sqrt{2}).
		\end{aligned}
	\end{equation}
	Then, we bound the term $A$.
	For simplicity, consider a general form of $A$ as
	\begin{align}
		f(x,y) = (x-\sqrt{2}y)^2 +\left(\frac{1}{x}-\frac{\sqrt{2}}{y}\right)^2,
	\end{align}
	and the following optimization problem,
	\begin{align}
		\min_{x,y} f(x,y), \text{ s.t. } x>0, y>0.
	\end{align} 
	
	By the first-order condition, we have
	\begin{align}\nonumber
		&\frac{\partial f(x,y)}{\partial x} \!=\! 2(x-\sqrt{2}y) + 2\left(\frac{1}{x}-\frac{\sqrt{2}}{y}\right)\cdot\left(-\frac{1}{x^2}\right)\!=\!0.\\\nonumber
		&\frac{\partial f(x,y)}{\partial y} \!=\! -2\sqrt{2}(x-\sqrt{2}y) \!-\! 2\sqrt{2}\left(\frac{1}{x}-\frac{\sqrt{2}}{y}\right)\!\cdot\!\left(-\frac{1}{y^2}\right)\!=\!0.
	\end{align}
	The solution is $x=y=1$.
	Then, we verify the second order condition.
	\begin{align}
		&\frac{\partial^2 f(x,y)}{\partial x^2} = 2 + \frac{12}{x^4},\\
		&\frac{\partial^2 f(x,y)}{\partial x \partial y} = -2\sqrt{2} - \frac{2\sqrt{2}}{x^2y^2},\\
		&\frac{\partial^2 f(x,y)}{\partial y^2} = 4 + 2\cdot\frac{12}{y^4},\\
		&\frac{\partial^2 f(x,y)}{\partial y\partial x} = -2\sqrt{2}-\frac{2\sqrt{2}}{x^2y^2}.	
	\end{align}
	Because the Hessian of $f(x,y)$ at $(1,1)$ is positive definite, $(1,1)$ is the solution and $f(1,1) = 2\cdot(1-\sqrt{2})^2 = 6-4\sqrt{2}$.
	So, we have
	\begin{align}\label{equ: A}
		A \ge 6 - 4\sqrt{2}.
	\end{align}
	Substituting (\ref{equ: A}) back to (\ref{equ:delta}), we have
	\begin{align}
		\Delta_1 - \Delta_2 > 0.
	\end{align}
	Thus, we have $f^\textnormal{case1} > f^\textnormal{case2}$, which contradicts with the assumption that $|\mathcal{S}^{+*}|>1$.
	This completes the proof.
\end{proof}
\begin{lemma}\label{lemma: S-}
	At the optimal solution of Problem \ref{prob: relaxed}, there is at most one clients in the set $\mathcal{S}^{-*}$.
\end{lemma}
The proof of Lemma \ref{lemma: S-} follows a similar structure to that of Lemma \ref{lemma: S+}. 
By assuming multiple clients in $\mathcal{S}^{-*}$ and deriving a contradiction through optimization sub-problems and KKT conditions, we can establish Lemma \ref{lemma: S-}.
\begin{lemma}\label{lemma: threshold}
	At the optimal solution of Problem \ref{prob: relaxed}, there exists thresholds $h_1 < h_2 < h_3$ such that
	\begin{align}
		k \in
		\begin{cases}
			\mathcal{S}^{+*}, & \text{if } v_k \le h_1, \\
			\mathcal{S}^{\textnormal{u}*}, &
			\text{if } h_1 < v_k \le h_2,\\
			\mathcal{S}^{-*}, &
			\text{if } h_2 < v_k \le h_3,\\
			\mathcal{S}^{0*}, &
			\text{if } h_3 < v_k.
		\end{cases}
	\end{align}
\end{lemma}
\begin{proof}
	We prove this lemma by establishing the ordering of the virtual costs across the optimal sets. Specifically, we show that the following three inequalities hold:
	\begin{itemize}
		\item For any client $i \in \mathcal{S}^{+*}$ and $j \in \mathcal{S}^{\textnormal{u}*}$, their virtual costs satisfy $v_i \le v_j$.
		\item For any client $j \in \mathcal{S}^{\textnormal{u}*}$ and $k \in \mathcal{S}^{-*}$, their virtual costs satisfy $v_j \le v_k$.
		\item For any client $k \in \mathcal{S}^{-*}$ and $l \in \mathcal{S}^{0*}$, their virtual costs satisfy $v_k \le v_l$.
	\end{itemize}
	Each of these statements will be proven individually by contradiction. 
	Throughout the proof, recall that the sets $\mathcal{S}^{+*}$ and $\mathcal{S}^{-*}$ contain at most one client.
	\subsubsection{Proof of $v_i \le v_j$ for all $i \in \mathcal{S}^{+*}$ and $j \in \mathcal{S}^{\textnormal{u}*}$}
	We proceed by contradiction. Assume there exists a client $i \in \mathcal{S}^{+*}$ and a client $j \in \mathcal{S}^{\textnormal{u}*}$ who violate the condition, such that their virtual costs satisfy $v_i > v_j$.
	
	To show this leads to a contradiction, we analyze a reallocation of resources between only these two clients, holding the allocations for all other clients fixed. 
	
	By definition of the sets $\mathcal{S}^{+*}$ and $\mathcal{S}^{\textnormal{u}*}$, we have $p_i > 1/N$ and $p_j = 1/N$. Now, consider any reallocation of their combined probability, $P = p_i + p_j$, that keeps the new probabilities above the $1/N$ threshold. For any such reallocation, the $L_1$-norm term of the objective function, $\|\boldsymbol{p}^\textnormal{s} - \boldsymbol{p}^\textnormal{u}\|_1$, remains constant. This allows us to find a contradiction by analyzing the subproblem concerned only with the remaining objective terms.
	
	Without the $L_1$-norm term, we only to consider the following minimization problem,
	\begin{equation}
		\begin{aligned}
			\min_{p_i, p_j, \epsilon_i,\epsilon_j}&  \frac{p_i^2}{\epsilon_i^2} + \frac{p_j^2}{\epsilon_j^2},\\
			\text{s.t. }& p_i+p_j = \tilde{P},\\
			&v_i\epsilon_i + v_j\epsilon_j = \tilde{B}.
		\end{aligned}
	\end{equation}
	
	As proved in Lemma \ref{lemma: S+}, we only need to consider the cases that one client has probability $1/N$.
	\begin{itemize}
	\item Case 1 (assumption): $p_i = 1/N$.
	In Case 1, the objective value is
	\begin{align}
		f^\text{case1} =  \frac{\left(\left(\oldfrac{1}{N^2}v_1^2\right)^{1/3}+\left((\tilde{P}-\oldfrac{1}{N})^2v_2^2\right)^{1/3}\right)^3}{\tilde{B}^2}.
	\end{align}

	\item Case 2: $p_j = 1/N$.
	In Case 2, objective value is 
	\begin{align}
		f^\text{case2} =  \frac{\left(\left(\oldfrac{1}{N^2}v_2^2\right)^{1/3}+\left((\tilde{P}-\oldfrac{1}{N})^2v_1^2\right)^{1/3}\right)^3}{\tilde{B}^2}.
	\end{align}
	\end{itemize}
	Because $\tilde{P} - \oldfrac{1}{N}\ge \oldfrac{1}{N}, v_1 < v_2$, we have
	\begin{align}
		f^\text{case2} < f^\text{case1}.
	\end{align}
	
	This contradicts with the assumption.
	
	\subsubsection{Proof of $v_j \le v_k$ for all $j \in \mathcal{S}^{\textnormal{u}*}$ and $k \in \mathcal{S}^{-*}$}
	We again use a proof by contradiction. Assume there exists a client $j \in \mathcal{S}^{\textnormal{u}*}$ and a client $k \in \mathcal{S}^{-*}$ whose virtual costs violate the stated condition, such that $v_j > v_k$.
	
	To derive a contradiction, we consider a reallocation of the combined resources between clients $j$ and $k$, holding all other client allocations fixed.
	
	By definition of the sets $\mathcal{S}^{\textnormal{u}*}$ and $\mathcal{S}^{-*}$, we have $p_j = 1/N$ and $p_k < 1/N$, respectively. Consequently, both allocations satisfy the condition $p \le 1/N$. Under any reallocation of their total probability, $P = p_j + p_k$, that preserves this condition for the new allocations, the $L_1$-norm term $\|\boldsymbol{p}^\textnormal{s} - \boldsymbol{p}^\textnormal{u}\|_1$ remains constant. This allows us to find a contradiction by analyzing the optimality of the subproblem that involves only the remaining objective terms.

	Without the L1 norm term, we only to consider the following minimization problem,
	\begin{equation}
		\begin{aligned}
			\min_{p_j, p_k, \epsilon_j,\epsilon_k}& \frac{p_j^2}{\epsilon_j^2} + \frac{p_k^2}{\epsilon_k^2},\\
			\text{s.t. }& p_j+p_k = \tilde{P},\\
			&v_j\epsilon_j + v_k\epsilon_k = \tilde{B}.
		\end{aligned}
	\end{equation}
	
	As proved in Lemma \ref{lemma: S-}, we only need to consider the cases that one client has probability $1/N$.
	\begin{itemize}
	\item Case 1: $p_1 = 1/N$:
	In this case, the objective value is
	\begin{align}
		f^\text{case1} =  \frac{\left(\left(\oldfrac{1}{N^2}v_1^2\right)^{1/3}+\left((\tilde{P}-\oldfrac{1}{N})^2v_2^2\right)^{1/3}\right)^3}{\tilde{B}^2}.
	\end{align}

	\item Case 2: $p_2 = 1/N$ (assumption).
	In this case, the objective value is 
	\begin{align}
		f^\text{case2} =  \frac{\left(\left(\oldfrac{1}{N^2}v_2^2\right)^{1/3}+\left((\tilde{P}-\oldfrac{1}{N})^2v_1^2\right)^{1/3}\right)^3}{\tilde{B}^2}.
	\end{align}
	Because $\tilde{P} - \oldfrac{1}{N}< \oldfrac{1}{N}, v_1 < v_2$, we have
	\begin{align}
		f^\text{case2} > f^\text{case1}.
	\end{align}
	\end{itemize}
	This contradicts with the assumption.
		
	\subsubsection{Proof of $v_k \le v_l$ for all $k \in \mathcal{S}^{-*}$ and $l \in \mathcal{S}^{0*}$}
	
	Finally, we prove the third inequality, again by contradiction. Assume there exists a client $k \in \mathcal{S}^{-*}$ and a client $l \in \mathcal{S}^{0*}$ for whom the condition is violated, such that their virtual costs are ordered $v_k > v_l$.
	
	We will show that this assumption allows for an exchange of the allocations between clients $k$ and $l$ that decreases the value of the objective function. Such a decrease would contradict the presumed optimality of the original solution. Let the original allocations be $(p_k, \epsilon_k)$ and $(p_l, \epsilon_l)$. Consider a new allocation where we simply swap their roles. We will now analyze the resulting change in the objective function.
	
	\begin{itemize}
	\item Case 1: $p_1 = p_k, p_2 =0$.
	In this case, the DP error term is
	\begin{align}
		f^\text{case1} = \frac{p_k^2}{B_k^2}v_1^2.
	\end{align}
	
	\item Case 2 (assumption): $p_2 = p_k, p_1=0$.
	In case 2,
	\begin{align}
		f^\text{case2} = \frac{p_k^2}{B_k^2}v_2^2.
	\end{align}
	Because $v_1 < v_2$, we have
	\begin{align}
		f^\text{case2} > f^\text{case1}.
	\end{align}
	\end{itemize}
	This contradicts with the assumption.
	This completes the proof.
\end{proof}

Combining Lemmas \ref{lemma: S+}, \ref{lemma: S-}, and \ref{lemma: threshold} together, we arrive at the Theorem \ref{theorem: structure}.

\subsection{Proof of Proposition \ref{proposition: blue}}\label{subsec: blue}
\begin{proof}
We denote the virtual cost of clients as $\boldsymbol{v} = [v_1, \cdots, v_i, \cdots, v_N], \boldsymbol{v}' = [v_1, \cdots, v_i', \cdots, v_N]$ ($\boldsymbol{v}$ and $\boldsymbol{v}'$ differs only in the $i^{th}$ entry).
We represent the optimal solutions under $\boldsymbol{v}$ and $\boldsymbol{v}'$ as $(\boldsymbol{p}^*,\boldsymbol{\epsilon}^*)$ and $(\boldsymbol{p}'^*,\boldsymbol{\epsilon}'^*)$, respectively.

Then, by the optimality of $(\boldsymbol{p}^*,\boldsymbol{\epsilon}^*)$ and $(\boldsymbol{p}'^*,\boldsymbol{\epsilon}'^*)$, we have
\begin{align}\label{equ: prove con 1}
	&\eta L(\boldsymbol{p}^*,\boldsymbol{\epsilon}^*) + \sum_{k=1}^{N} v_k \epsilon_k^* \le \eta L(\boldsymbol{p}'^*,\boldsymbol{\epsilon}'^*) + \sum_{k=1}^{N} v_k \epsilon_k'^*,  \\\label{equ: prove con 2}
	&\eta L(\boldsymbol{p}'^*,\boldsymbol{\epsilon}'^*) \!+\! \sum_{k\neq i} v_k \epsilon_k'^* \!+\! v_i'\epsilon_i'^* \!\le\! \eta L(\boldsymbol{p}^*,\boldsymbol{\epsilon}^*) \!+\! \sum_{k\neq i} v_k \epsilon_k^* \!+\! v_i'\epsilon_i^*.
\end{align}
Combining (\ref{equ: prove con 1}) and (\ref{equ: prove con 2}) yields
\begin{equation}
\begin{aligned}
	v_i' \epsilon_i'^* - v_i' \epsilon_i^* &\le\eta L(\boldsymbol{p}^*,\boldsymbol{\epsilon}^*) \!-\! \eta L(\boldsymbol{p}'^*,\boldsymbol{\epsilon}'^*) \!+\! \sum_{k\neq i} v_k \epsilon_k^* \!-\! \sum_{k\neq i} v_k \epsilon_k'^*\\
	&\le v_i \epsilon_i'^* - v_i \epsilon_i^*.
\end{aligned}
\end{equation}
Thus, we have
\begin{align}\label{equ: prove con 3}
	(v_i'-v_i)(\epsilon_i'^* - \epsilon_i^*) \le 0.
\end{align}
By the inequality (\ref{equ: prove con 3}) and Assumption \ref{assumption: increasing}, we can conclude that $\epsilon_k$ is weakly decreasing with $c_k$.
This completes the proof.
\end{proof}

\subsection{Proof of Theorem \ref{theorem: budgets}}\label{app: budgets}
\begin{proof} 
 When the client sets $\mathcal{S}^+,\mathcal{S}^\text{u},\mathcal{S}^-,\mathcal{S}^0$ are fixed, by KKT conditions, for any client $k\in\mathcal{N}$, we have 
	\begin{align}\label{equ: aux 1}
		\frac{(p_{k}^\text{s})^2}{\epsilon_k^3} = c_3 v_k, 
	\end{align}
	where $c_3$ is a constant.
	Additionally, given budget $B$, we have
	\begin{align}\label{equ: aux 2}
		c_3 = \frac{(\sum_{i=1}^{N}v_i^{2/3}(p_i^\text{s})^{2/3})^3}{B^3}.
	\end{align}
	Substituting (\ref{equ: aux 2}) into (\ref{equ: aux 1}) yields Theorem \ref{theorem: budgets}.
	This completes the proof.
\end{proof}

\end{document}